\documentclass{article}
\PassOptionsToPackage{style=apa}{biblatex}
\usepackage[arxiv]{kf-basic}
\usepackage[mode=tex,subpreambles,sort]{standalone}

\zcsetup{lang=english,cap=true,abbrev=false}
\zcRefTypeSetup{equation}{
  Name-sg={Eq.},
  Name-pl={Eqs.},
}

\title{Provable Data Scaling Law for Meta Learning via Complexity Minimization}
\author[1,2]{Kazuto Fukuchi\thanks{fukuchi@cs.tsukuba.ac.jp, Corresponding author}}
\author[3,4]{Ryuichiro Hataya}
\author[4,5,6]{Kota Matsui}
\affil[1]{University of Tsukuba, Japan}
\affil[2]{RIKEN AIP, Japan}
\affil[3]{SB Intuitions Corp., Japan}
\affil[4]{Kyoto University, Japan}
\affil[5]{Shiga University, Japan}
\affil[6]{Institute of Science Tokyo, Japan}

\begin{document}
\sloppy
\hbadness=10000

\maketitle
\begin{abstract}
  Pre-training has become a fundamental paradigm in modern machine learning, with one of its key empirical benefits being reduced downstream sample complexity as the scale of pre-training data increases. However, existing theoretical frameworks for pre-training do not fully explain this phenomenon. In this paper, we introduce complexity minimization, a novel meta-representation learning framework designed to enable theoretical analysis of this scaling behavior, which learns representations by evaluating the downstream model complexity best suited to each domain and minimizing the worst-case such complexity across source domains. Our end-to-end theoretical analysis, spanning pre-training through downstream regression, shows that this framework provably captures this scaling behavior; in particular, we show that the error rate of few-shot adaptation improves as the amount of meta-training data grows. Empirically, we demonstrate that incorporating complexity regularization into existing meta-learning methods consistently improves downstream sample efficiency.
\end{abstract}

\section{Introduction}
Pre-training, encompassing self-supervised learning, representation learning, and meta-learning, is now a fundamental component of modern machine learning, as demonstrated by the recent success of foundation models, large models pre-trained on massive datasets. 
For example, in natural language processing, large-scale pre-trained
language models such as BERT and GPT-3 have shown that representations
learned from broad text corpora can be adapted to a wide range of downstream
tasks, including natural language inference, question answering, and text
generation~\autocite{devlin2019bert,brown2020language}. 
In computer vision and vision--language learning, contrastive and promptable pre-training has led to models such as CLIP and the Segment Anything Model, which exhibit strong transferability across image classification, retrieval, and segmentation tasks with little or no task-specific supervision \autocite{radford2021learning,kirillov2023segment}. 
A similar paradigm has also emerged in scientific and domain-specific applications: AlphaFold has transformed protein-structure prediction by leveraging large-scale biological sequence and structural information \autocite{jumper2021highly}, while medical large language models such as Med-PaLM illustrate the potential of pre-training and instruction tuning for clinical question answering~\autocite{singhal2023large}.  
In robotics, large transformer-based policies trained on diverse robot interaction data have shown improved generalization to new objects, environments, and instructions~\autocite{brohan2022rt}. 

The theoretical study of pre-training, including analyses for few-shot learning~\autocite{duFewShotLearningLearning2020}, in-context learning~\autocite{baiTransformersStatisticiansProvable2023,liTransformersAlgorithmsGeneralization2023,kimTransformersAreMinimax2024}, and meta-learning~\autocite{collinsMAMLANILProvably2022,aliakbarpourMetalearningVeryFew2024,huang2022provable,wuLearningLearnContrastive2025,dingStabilityGeneralizationMetaLearning2025,blockProvableMetaLearningLowRank2025}, has revealed the advantage of pre-training in terms of the sample complexity of the downstream learning task. For example, \textcite{duFewShotLearningLearning2020} showed that the existence of a common linear representation shared across source and downstream tasks yields a reduction in sample complexity. Pre-training has also been shown to reduce downstream sample complexity in in-context learning under generalized linear models~\autocite{baiTransformersStatisticiansProvable2023}, nonparametric regression models~\autocite{kimTransformersAreMinimax2024}, and a hypothesis class with bounded algorithmic stability~\autocite{liTransformersAlgorithmsGeneralization2023}. Furthermore, many researchers have shown that the meta-learning algorithms provably reduce the downstream sample complexity~\autocite{collinsMAMLANILProvably2022,aliakbarpourMetalearningVeryFew2024,huang2022provable,wuLearningLearnContrastive2025,dingStabilityGeneralizationMetaLearning2025,blockProvableMetaLearningLowRank2025}.

These results, however, are inconsistent with the empirical phenomenon known as the {\em scaling law}, first introduced by \textcite{kaplanScalingLawsNeural2020}. The {\em data scaling law} for pre-trained models, in particular, shows that pre-training on more data leads to better error rate of the downstream learning task~\autocite{henighanScalingLawsAutoregressive2020, mikamiScalingLawSyn2real2023}. The aforementioned theoretical results cannot explain this empirical finding, since the downstream error rates they establish are independent of the pre-training data size.

Recently, \textcite{fukuchiProvableTargetSample2026} has provided a theoretical framework that can explain the data scaling law for pre-trained models. Their framework, caulking, adapts the pre-trained model to the downstream task by inserting an adapter, as in parameter-efficient fine-tuning (PEFT) methods. Their analysis establishes that training the pre-trained model so that the complexity of the adapter decreases as the pre-training data size grows provably reduces the sample complexity of the downstream task.

However, their results lack an end-to-end analysis from pre-training to downstream learning, leaving unclear the training strategy that achieves the data scaling law. Developing a pre-training algorithm that provably achieves the data scaling law is important not only for the theoretical understanding of recent advances in foundation models, but also for guiding the practical development of pre-trained models.

\paragraph{Our contribution}
The main contribution of this paper is a meta-representation learning algorithm together with its theoretical analysis, proving the achievability of the data scaling law. Our contributions are summarized as follows:
\begin{itemize}
  \item We propose \emph{complexity minimization}, a novel meta-representation learning framework that selects a feature extractor by minimizing the worst-case best model complexity across observed source domains. The best model complexity serves as a proxy for the convergence rate of the downstream excess error: for instance, when the underlying regression function is sparse, the sparsity level governs the downstream convergence rate, so minimizing it directly reduces downstream sample complexity.
  \item To instantiate complexity minimization, we construct a novel estimator of the best model complexity using Lepski's method~\autocite{lepskiiProblemAdaptiveEstimation1991}, a principled adaptive model selection procedure that identifies the optimal complexity level from a sample without prior knowledge of the underlying complexity parameters.
  \item We provide an end-to-end theoretical analysis spanning meta-training through downstream learning and prove that the downstream error rate exhibits the data scaling law. Specifically, the downstream excess error achieves the rate
    \begin{align}
      (n/\ln n)^{-\beta^* + O(\ln^{-\gamma} m)}, \label{eq:scaling-law-rate}
    \end{align}
    where $m$ and $n$ are the meta-learning and downstream sample sizes, respectively, $\beta^* > 0$ is the ideal downstream convergence exponent, and $\gamma > 0$ is a constant. The exponent in \zcref{eq:scaling-law-rate} approaches $\beta^*$ as $m \to \infty$, meaning the downstream error decays faster with $n$ as the meta-training sample size grows, which is precisely the data scaling law.
  \item We empirically verify that adding a norm-based complexity regularizer to standard meta-learning algorithms consistently improves downstream sample efficiency across multiple baselines and datasets.
\end{itemize}

All missing proofs are left to the appendix.

\section{Problem Formulation}

\paragraph{Notation}
For a positive integer $m$, let $[m] = \Bab{1, \dots, m}$. For sequences $a_n$ and $b_n$, we write $a_n \lesssim b_n$ (resp.\ $a_n \gtrsim b_n$) if $a_n \le Cb_n$ (resp.\ $a_n \ge Cb_n$) for some $C > 0$ and all $n$; $a_n \asymp b_n$ means both hold. For a vector $x \in \RealSet^d$ and a function $f \colon \mathcal{X} \to \RealSet$, we write $\Vab*{x}_p$ for the $\ell^p$-norm and $\Vab*{f}_{L^p}$ for the $L^p$-norm. We write $\p$ and $\Mean$ for probability and expectation. For a measurable function $f \colon \mathcal{X} \to \RealSet$ and a random variable $X$ taking values in $\mathcal{X}$, we define $\Vab*{f}_{L^p(X)} = (\Mean[|f(X)|^p])^{1/p}$ for $p \in [1,\infty)$. For a set $A$ endowed with a metric $\rho$ and $\epsilon > 0$, $N(\epsilon, A, \rho)$ denotes the $\epsilon$-covering number of $A$; for $A$ endowed with a norm $\Vab*{\cdot}$, we write $N(\epsilon, A, \Vab*{\cdot}) = N(\epsilon, A, \rho_{\Vab*{\cdot}})$ where $\rho_{\Vab*{\cdot}}(x,y) = \Vab*{x-y}$. Additional notation used in the proofs is collected in \zcref{sec:analyses}.

\paragraph{Meta-representation learning problem}
Consider a representation learning problem with samples from multiple domains. Let $\mathcal{P}^*$ be the set of all pairs $(X, Y)$ of random variables, where $X \in \mathcal{X} \subseteq \RealSet^d$ is a feature and $Y \in \RealSet$ is an outcome. Throughout, we assume $\Mean[Y|X] \in [0,1]$ almost surely. Let $\mathcal{P} \subseteq \mathcal{P}^*$ denote the subset of feature-outcome pairs associated with all domains of interest, and let $\mathfrak{P} \subseteq 2^{\mathcal{P}^*}$ be the collection of all possible realizations of $\mathcal{P}$. The learner knows $\mathcal{P}^*$ but not $\mathcal{P}$. Let $(X^{(1)}, Y^{(1)}), \dots, (X^{(D)}, Y^{(D)}) \in \mathcal{P}$ be the feature-outcome pairs for $D$ observed domains, drawn i.i.d. from a distribution over $\mathcal{P}$. The learner observes $m$ i.i.d.\ copies of each $(X^{(d)}, Y^{(d)})$, denoted $(X_1^{(d)}, Y_1^{(d)}), \dots, (X_m^{(d)}, Y_m^{(d)})$. The goal is to learn a feature extractor $\phi \colon \mathcal{X} \to \RealSet^p$ that minimizes the sample complexity of learning a regressor of the form $f \circ \phi$ for some $f \colon \RealSet^p \to [0,1]$ from an additional sample drawn from some $(X, Y) \in \mathcal{P}$, which we refer to as the downstream learning task.

\begin{remark}[Intuition behind $\mathcal{P}^*$ and $\mathcal{P}$]
  The distinction between $\mathcal{P}^*$ and $\mathcal{P}$ is central to characterizing the conditions for successful meta representation learning. We assume that a single common feature representation performs well for all domains in $\mathcal{P}$, so that obtaining such a representation minimizes the sample complexity of the downstream task. In other words, $\mathcal{P}$ shares a common feature representation, whereas $\mathcal{P}^*$ encompasses all possible feature-outcome distributions over a variety of representations. Since the learner does not know $\mathcal{P}$, they do not know this common representation either. Identifying $\mathcal{P}$ from pre-training data is therefore valuable for reducing downstream sample complexity.
\end{remark}

\paragraph{Downstream regression problem}
In the downstream task, the learner receives an additional sample from some $(X, Y) \in \mathcal{P}$ and a pre-trained feature extractor $\phi \colon \mathcal{X} \to \RealSet^p$, and finds a head function $f \colon \RealSet^p \to [0,1]$ such that $f \circ \phi$ is an accurate regressor for $(X, Y)$. Let $(X_1, Y_1), \dots, (X_n, Y_n)$ be $n$ i.i.d.\ copies of $(X, Y)$. The quality of $f$ is measured by the expected squared error
\begin{align}
  E_\phi(f; X, Y) = \Mean\ab[\ab((f \circ \phi)(X) - Y)^2].
\end{align}
Equivalently, $f$ minimizes the excess error $\bar{E}_\phi(f; X, Y) = \Vab*{f \circ \phi - \Mean[Y | X]}_{L^2(X)}^2$, where $\Mean[Y | X]$ is the Bayes optimal regressor for $(X, Y)$. Let $f_n$ denote the head function learned from the sample. The sample complexity of the downstream task is characterized by the rate at which $\bar{E}_\phi(f_n; X, Y)$ decreases as $n$ grows.

\section{Complexity Minimization}\label{sec:complexity-minimization}
\begin{figure}
  \centering
  \begin{tikzpicture}[
  font=\footnotesize,
  >=Latex,
  arr/.style={-{Latex[length=1.35mm]}, thin, draw=gray!70},
  lbl/.style={anchor=east, align=right, inner sep=0pt, outer sep=0pt},
  domainbox/.style={
    rounded corners=1.5pt,
    draw=gray!60,
    fill=gray!6,
    align=center,
    inner xsep=2pt,
    inner ysep=1.2pt,
  },
  topbox/.style={
    rounded corners=2pt,
    draw=gray!50,
    fill=blue!4,
    align=center,
    inner xsep=3pt,
    inner ysep=2pt,
    minimum width=6.15cm,
  },
]

  \begin{scope}[local bounding box=cmfig]
    \node[topbox] (ml) {%
      $\displaystyle \min_{\phi} \; \max_{d=1,\ldots,D} J^{\dagger(d)}_n(\phi)$};

    \path let
      \p1 = ($(ml.south east)-(ml.south west)$),
      \p2 = ($(ml.north west)-(ml.south west)$),
      \n6 = {max(2.2, 0.0065*abs(\x1))},
      \n7 = {max(7, 0.012*abs(\x1))},
      \n2 = {(abs(\x1) - 2*\n6 - \n7)/2},
      \n3 = {max(5.5, 0.09*abs(\y2))},
      \n4 = {max(14, \n2 - 4)},
      \n5 = {max(21, 0.16*abs(\y2))}
    in
      node[domainbox,
        minimum width=\n2,
        minimum height=\n5,
        text width=\n4,
        anchor=north west] (d1) at ($ (ml.south west) + (0,-\n3) $) {%
          Domain 1\\ $J^{\dagger(1)}_n(\phi)$}
      node[anchor=center, inner xsep=2pt, inner ysep=0.35pt] (dd) at ($ (d1.east) + (\n6 + \n7/2, 0) $) {$\cdots$}
      node[domainbox,
        minimum width=\n2,
        minimum height=\n5,
        text width=\n4,
        anchor=north west] (dD) at ($ (d1.north east) + (\n6 + \n7 + \n6, 0) $) {%
          Domain $D$\\ $J^{\dagger(D)}_n(\phi)$};

    \coordinate (role-label-east) at ([xshift=-0.12cm]ml.west);
    \coordinate (base-row-mid) at ($(d1.west)!0.5!(dD.west)$);
    \node[lbl] (lmeta) at (role-label-east |- ml.center) {\textbf{Meta-learner}};
    \node[lbl] (lbase) at (role-label-east |- base-row-mid) {\textbf{Base-learners}};

    \draw[arr] ($(d1.north)-(0,0.4pt)$) -- ($(ml.south -| d1)+(0,0.4pt)$);
    \draw[arr] ($(dD.north)-(0,0.4pt)$) -- ($(ml.south -| dD)+(0,0.4pt)$);
  \end{scope}

  \coordinate (row-south-mid) at ($(d1.south)!0.5!(dD.south)$);
  \coordinate (fig-h-center) at ($(cmfig.west)!.5!(cmfig.east)$);

\end{tikzpicture}
  \caption{Conceptual diagram of \emph{complexity minimization}. Each base-learner (bottom) assesses the best model complexity $J^{\dagger(d)}_n(\phi)$ for its domain $d$ and reports it to the meta-learner. The meta-learner (top) selects $\phi$ to minimize the worst-case best model complexity over the observed domains, thereby reducing the downstream sample complexity across all domains.}
  \label{fig:complexity-minimization}
\end{figure}
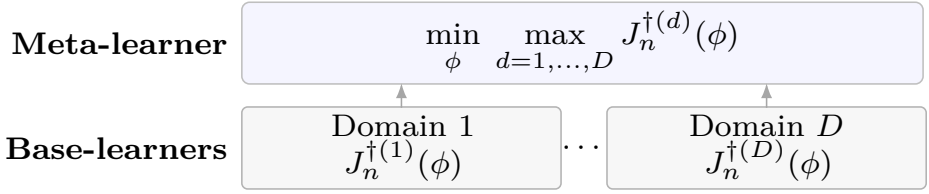

We propose \emph{complexity minimization} (\zcref{fig:complexity-minimization}), a meta-representation learning framework following the meta-learner/base-learner architecture of existing approaches~\autocite{hospedalesMetaLearningNeuralNetworks2022,finnModelAgnosticMetaLearningFast2017,wuLearningLearnContrastive2025,dingStabilityGeneralizationMetaLearning2025}. The meta-learner maintains the feature extractor $\phi$ as its meta-parameter~(\zcref{fig:complexity-minimization} top); each base-learner is associated with one observed domain and evaluates $\phi$ by a domain-specific criterion~(\zcref{fig:complexity-minimization} bottom). Many existing meta-learning algorithms, including MAML~\autocite{finnModelAgnosticMetaLearningFast2017}, instantiate this criterion as the downstream regression error.

Our key departure is to replace the regression error with the \emph{best model complexity} of the head function, which serves as a proxy for the convergence rate of the downstream excess error. Formally, let $\mathcal{F}_J$ be a sequence of increasing classes of head functions $f \colon \RealSet^p \to [0,1]$, indexed by a complexity parameter $J \in \NaturalSet$, where model complexity (e.g., the number of non-zero weights) increases with $J$. Letting $f_{n,J}$ denote the head function learned over $\mathcal{F}_J$ from a downstream sample of size $n$, the best model complexity for domain $d$ under $\phi$ is
\begin{align}
  J^{\dagger(d)}_n(\phi) = \argmin_{J \in \NaturalSet} \bar{E}_\phi(f_{n,J}; X^{(d)}, Y^{(d)}).
\end{align}
When, for instance, the head function is truly sparse, the minimal sufficient sparsity level governs the downstream convergence rate; a smaller best model complexity therefore implies faster downstream learning.

The meta-learner collects these criteria from every base-learner and selects $\phi$ to minimize the worst-case value across all observed domains:
\begin{align}
  \min_{\phi \in \Phi} \max_{d=1,\ldots,D} J^{\dagger(d)}_n(\phi), \label{eq:complexity-minimization}
\end{align}
where $\Phi$ is a class of feature extractors. Because $J^{\dagger(d)}_n(\phi)$ depends on the unknown downstream distribution, it must be estimated from pre-training samples in practice. This estimation step is precisely what connects complexity minimization to the data scaling law: larger pre-training samples yield more accurate complexity estimates, resulting in a smaller selected model complexity and therefore a faster downstream convergence rate across all domains in $\mathcal{P}$.

\section{Provable Scaling Laws via Complexity Minimization}

In this section, we present a concrete instantiation of the complexity minimization framework and establish the data scaling law for the resulting algorithm.

\paragraph{Downstream rate}
We employ a specific characterization of the downstream error rate $\bar{E}_\phi(f_{n,J}; X, Y)$ to build the concrete algorithm. Specifically, the downstream error is characterized by two quantities: the approximation error and the Minkowski–Bouligand dimension. This characterization is applicable to deep neural network based estimators~\autocite{schmidt-hieberNonparametricRegressionUsing2020,imaizumiGeneralizationBoundsDeep2023,suzukiAdaptivityDeepReLU2019,nishimuraMinimaxOptimalityConvolutional2023,hayakawaMinimaxOptimalitySuperiority2020,kohlerRateConvergenceFully2021,chenNonparametricRegressionLowdimensional2022} and hence covers modern machine learning algorithms.

We first introduce these two quantities and then present the characterization on the downstream error. Given $\phi \in \Phi$ and $(X, Y) \in \mathcal{P}^*$, the approximation error of the regression function under $\phi$ is
\begin{align}
  A_J(\phi; X, Y) = \min_{f_J \in \mathcal{F}_J} \Vab{f_J \circ \phi - \Mean[Y | X]}_{L^2(X)}^2.
\end{align}
The Minkowski–Bouligand dimension of a set $S$ with respect to a norm $\Vab*{\cdot}$ is defined as
\begin{align}
  d_{\mathrm{MB}}(S; \Vab{\cdot}) = \limsup_{\epsilon \to 0} \frac{\ln N(\epsilon, S, \Vab{\cdot})}{\ln(1/\epsilon)}.
\end{align}
Building from these definitions, we have the following theorem.
\begin{theorem}[Based on \textcite{schmidt-hieberNonparametricRegressionUsing2020,hayakawaMinimaxOptimalitySuperiority2020}]\label{thm:oracle-rate}
  Fix $(X, Y) \in \mathcal{P}^*$ and $\phi \in \Phi$. Let $\mathcal{F}_J$ be a sequence of increasing classes of functions $f: \RealSet^p \to [0,1]$ such that $d_{\mathrm{MB}}(\mathcal{F}_J; \Vab{\cdot}_{L^\infty}) \le J$ for any $J \in \NaturalSet$. Then, there is a learning algorithm $f_{n,J}$ such that with high probability,
  \begin{align}
    \bar{E}_\phi(f_{n,J}; X, Y) \lesssim A_J(\phi; X, Y) + \frac{J\ln(n)}{n}. \label{eq:oracle-rate}
  \end{align}
\end{theorem}
The convergence rate induced from \zcref{thm:oracle-rate} is determined by the decreasing rate of the approximation error as $J$ grows. For example, if $A_J(\phi; X, Y) \asymp J^{-2\alpha}$ for some $\alpha > 0$, then the convergence rate is $\asymp (n/\ln n)^{-2\alpha/(2\alpha + 1)}$ with $J \asymp (n/\ln n)^{1/(2\alpha + 1)}$, derived by optimizing the right hand side of \zcref{eq:oracle-rate} for $J$. As the choice of $J$ depends on the unknown parameter $\alpha$, we refer to this rate as the oracle rate.

\paragraph{Ideal downstream rate}
We introduce the ideal decreasing rate of the downstream error. In our analysis, we focus only on the polynomial decreasing rate of $A_J$.
\begin{assumption}[Polynomial decreasing rate of $A_J$]\label{asm:poly-approx}
  There exists a functional $\alpha(\phi; X, Y) \in (0,\infty)$ for $(X, Y) \in \mathcal{P}^*$ and $\phi \in \Phi$ such that $A_J(\phi; X, Y) \asymp J^{-2\alpha(\phi; X, Y)}$. Additionally, there exists a constant $\bar{\alpha} < \infty$ such that $\alpha(\phi; X, Y) \le \bar{\alpha}$ for any $\phi \in \Phi$ and $(X, Y) \in \mathcal{P}^*$.
\end{assumption}
We write $\beta(\phi; X, Y) = 2\alpha(\phi; X, Y)/(2\alpha(\phi; X, Y)+1)$ and $\bar\beta = 2\bar{\alpha}/(2\bar\alpha + 1)$. From \zcref{thm:oracle-rate}, $\bar{E}_\phi(f_{n,J}; X, Y) \lesssim (n/\ln n)^{-\beta(\phi; X, Y)}$ with appropriately chosen $J$ for fixed $\phi$ and $(X, Y) \in \mathcal{P}$. The ideal convergence exponent for a given $\mathcal{P} \in \mathfrak{P}$ is therefore
\begin{align}
  \beta_{\mathcal{P}}^* \coloneqq \sup_{\phi \in \Phi}\inf_{(X, Y) \in \mathcal{P}}\beta(\phi; X, Y).
\end{align}

\paragraph{Technical assumptions}
For our main theorem, we need several technical assumptions. First, we introduce an assumption about the complexity of the class of feature extractors, $\Phi$.
\begin{assumption}[Complexity of $\Phi$]\label{asm:complex-phi}
  There exist $\beta_0 > 0$ with $\beta_0 + \bar\beta < 1$ and $\gamma \in (0,1]$ such that for any $\mathcal{P} \in \mathfrak{P}$, for all $m \ge 1$,
  \begin{align}
    \sum_{J \in [m] : J \le \frac{m}{\ln m}}N(V_{m,J}, \Phi, \rho_J) \lesssim m^{(m/\ln m)^{\beta_0}},
  \end{align}
  and
  \begin{align}
    \ln N(\ln^{-\gamma}m, \Phi, \rho_{\beta,\mathcal{P}}) \lesssim \ln m,
  \end{align}
  where $\rho_J(\phi, \phi') = \sup_{f \in \mathcal{F}_J}\Vab*{f\circ\phi - f\circ\phi'}_{L^\infty}$ and $\rho_{\beta,\mathcal{P}}(\phi, \phi') = \sup_{(X,Y) \in \mathcal{P}}|\beta(\phi; X, Y) - \beta(\phi'; X, Y)|$.
\end{assumption}
\zcref{asm:complex-phi} requires that two types of metric entropy conditions on $\Phi$ hold simultaneously. Constructing concrete families $\Phi$ satisfying \zcref{asm:complex-phi} is an important open problem.

Next, we introduce an assumption about the distribution over the domains.
\begin{assumption}[Uniform domain sampling]\label{asm:unif-domain}
  For each $\mathcal{P} \in \mathfrak{P}$, $(X, Y) \in \mathcal{P}$ is distributed by the domain distribution. There exists $\nu > 0$ such that for any $\delta > 0$ and any $\phi \in \Phi$, the domain-distributed $(X', Y')$ satisfies $|\beta(\phi; X, Y) - \beta(\phi; X', Y')| \le \delta$ with probability at least $C\delta^\nu$ for some universal constant $C > 0$.
\end{assumption}
Since the learner observes samples from only finitely many domains, these domains must collectively represent all of $\mathcal{P}$ for the meta-learned representation to generalize. \zcref{asm:unif-domain} imposes a near-uniform condition on the domain distribution, ensuring that the observed domains approximately cover $\mathcal{P}$ in terms of $\beta$ when $D$ is sufficiently large.

Lastly, we introduce a mild assumption about the noise in $Y$.
\begin{assumption}[Sub-gaussian noise]\label{asm:sub-gauss-noise}
  For any $(X, Y) \in \mathcal{P}^*$, conditioned on $X$, $Y - \Mean[Y|X]$ is sub-gaussian with variance proxy at most $\sigma^2$; namely, for all $\lambda \in \RealSet$, $\Mean[\exp(\lambda(Y - \Mean[Y|X]))|X] \le e^{\sigma^2\lambda^2/2}$ almost surely.
\end{assumption}
\zcref{asm:sub-gauss-noise} is a standard assumption employed in a broad literature~\autocite{schmidt-hieberNonparametricRegressionUsing2020,imaizumiGeneralizationBoundsDeep2023,suzukiAdaptivityDeepReLU2019,nishimuraMinimaxOptimalityConvolutional2023,hayakawaMinimaxOptimalitySuperiority2020,kohlerRateConvergenceFully2021,chenNonparametricRegressionLowdimensional2022}.

\paragraph{Main result}
We present our main theoretical result, exhibiting meta-learning and downstream algorithms that provably achieve the data scaling law.
\begin{theorem}[Main theorem]\label{thm:main}
  Assume \zcref{asm:poly-approx,asm:complex-phi,asm:sub-gauss-noise,asm:unif-domain}. There exist a meta-learning algorithm and a downstream learning algorithm such that, if $\ln^{1+\nu\gamma}m \lesssim D$ and $\ln D \lesssim \ln m$, then with probability at least $1 - O(m^{-1}) - O(n^{-1})$,
  \begin{align}
    \bar{E}_{\hat\phi}(\hat{f}_n) \lesssim \ab(\frac{n}{\ln(n)})^{-\beta^*_{\mathcal{P}}+O(\frac{1}{\ln^\gamma m})}.
  \end{align}
\end{theorem}
The error rate in \zcref{thm:main} exhibits the data scaling law: the rate at which the downstream error decreases in $n$ improves as the meta-learning sample size $m$ grows. To the best of our knowledge, this is the first end-to-end theoretical analysis establishing the data scaling law for pre-training. The proof constructs concrete meta-learning and downstream algorithms and establishes that both achieve the rate stated in \zcref{thm:main}. The analyses of our meta-learning and downstream algorithms for proving \zcref{thm:main} are found in \zcref{sec:analyses}.

\subsection{Base-learner}
Our meta-learning algorithm consists of interacting base-learner and meta-learner as described in \zcref{sec:complexity-minimization}, and in this subsection, we describe the concrete construction of the base-learner. In the complexity minimization framework, the base-learner assesses the best model complexity $J^{\dagger(d)}_n(\phi)$. Because this quantity depends on the unknown downstream distribution, each base-learner must estimate it from pre-training samples. Our idea in estimating the best model complexity is to employ the {\em Lepski's method}~\autocite{lepskiiProblemAdaptiveEstimation1991}, which selects the model complexity adaptively without knowledge of the underlying complexity parameters.

\paragraph{Lepski's method}
Lepski's method~\autocite{lepskiiProblemAdaptiveEstimation1991} is a powerful tool for adaptive model selection in nonparametric statistics and can automatically find the optimal model complexity from a sample without prior knowledge of the underlying complexity parameters. For example, it builds estimators for nonparametric regression within smooth function classes such as H\"older, Sobolev, and Besov spaces, achieving convergence rates determined by the smoothness parameter without prior knowledge of it~\autocite{lepskiiProblemAdaptiveEstimation1991,lepskiOptimalSpatialAdaptation1997}.

We now instantiate Lepski's method using the oracle rate from \zcref{thm:oracle-rate} with fixed $\phi$ and $(X, Y) \in \mathcal{P}^*$. The idea is to estimate $A_J(\phi; X, Y)$ and select $J$ so that the estimated $A_J$ and the term $\frac{J\ln(n)}{n}$ are balanced. Let $f^*_{Y|X,\phi,J}$ be the best regressor in $\mathcal{F}_J$ such that
\begin{align}
  \Vab{f^*_{Y|X,\phi,J} \circ \phi - \Mean[Y | X]}_{L^2(X)}^2 = \min_{f_J \in \mathcal{F}_J} \Vab{f_J \circ \phi - \Mean[Y | X]}_{L^2(X)}^2,
\end{align}
where $f^*_{Y|X,\phi,J}$ is an arbitrary one if the tie occurs. We omit the first and second subscripts to denote $f^*_J$ if $(X, Y)$ and $\phi$ are clear from the context. Then, the algorithm selects $J$ following the Lepski's rule, defined as
\begin{align}
  J^*_n(\phi; X, Y) = \min\Bab{J \in [m] : \forall J \le J' \le \frac{m}{\ln m}, \Vab{f^*_J \circ \phi - f^*_{J'} \circ \phi}_{L^2(X)}^2 \le \rho V_{n,J'}}, \label{eq:m-dim-x-y}
\end{align}
where $V_{n,J} = \frac{J\ln(n)}{n}$ is referred to as a variance term or a majorant, and $\rho > 0$ is a constant chosen as specified in the analyses. The intuition behind the Lepski's rule is that if $f^*_J$ sufficiently approximates the regression function, then increasing $J$ does not significantly deviate $f^*_J$ up to the variance.

\paragraph{Empirical estimation}
The base-learner for domain $d$ estimates the Lepski rule value $J^*_m(\phi; X^{(d)}, Y^{(d)})$ from \zcref{eq:m-dim-x-y} as a proxy for the best model complexity $J^{\dagger(d)}_n(\phi)$ from \zcref{eq:complexity-minimization}. For each $J$, a sieved least-squares estimator is computed, yielding the regressor
\begin{align}
  \hat{f}^{(d)}_{m,J} = \argmin_{f_J \in \mathcal{F}_J} \frac{1}{m} \sum_{i=1}^{m} \ab((f_J \circ \phi)(X_i^{(d)}) - Y_i^{(d)})^2.
\end{align}
Then, the estimated complexity is obtained as
\begin{multline}
  \hat{J}^{(d)}(\phi) = \min\Bab[Bigg]{J \in [m] : \forall J \le J' \le \frac{m}{\ln m}, \\ \frac{1}{m}\sum_{i=1}^m\ab((\hat{f}^{(d)}_{m,J} \circ \phi)\ab(X^{(d)}_i) - (\hat{f}^{(d)}_{m,J'} \circ \phi)\ab(X^{(d)}_i))^2 \le \rho V_{m,J'}}.
\end{multline}

\subsection{Meta-learner}\label{sec:meta-learner}
The meta-learner collects the estimated complexity $\hat{J}^{(d)}(\phi)$ from all $D$ base-learners and selects $\phi$ to minimize the worst-case estimated complexity, forming the empirical counterpart of \zcref{eq:complexity-minimization}. Specifically, the estimated feature extractor is defined as
\begin{align}
  \hat\phi = \argmin_{\phi \in \Phi} \max_{d \in [D]} \hat{J}^{(d)}(\phi). \label{eq:meta-feature-extractor}
\end{align}
Minimizing the worst case ensures that $\hat\phi$ simultaneously reduces the estimated downstream complexity across all observed domains, yielding a feature extractor whose downstream performance generalizes uniformly over $\mathcal{P}$.

\subsection{Downstream Algorithm}\label{sec:downstream-algorithm}
At downstream time, we again carry out Lepski's method to construct the learned regressor. Specifically, define
\begin{align}
  \hat{f}_{n,J} = \argmin_{f_J \in \mathcal{F}_J} \frac{1}{n} \sum_{i=1}^{n} \ab((f_J \circ \hat\phi)(X_i) - Y_i)^2.
\end{align}
Then, the complexity is estimated as
\begin{align}
  \hat{J}_n = \min\Bab{J \in [n] : \forall J \le J' \le \frac{n}{\ln n}, \frac{1}{n}\sum_{i=1}^n\ab((\hat{f}_{n,J} \circ \hat\phi)\ab(X_i) - (\hat{f}_{n,J'} \circ \hat\phi)\ab(X_i))^2 \le \rho V_{n,J'}}.
\end{align}
Consequently, the learned regressor is given by $\hat{f}_n = \hat{f}_{n,\hat{J}_n}$.

\section{Experiments}

We empirically validate the effect of complexity minimization by adding a complexity regularization term to meta-learning algorithms.
The experiments demonstrate how the pre-training size $m$ affects the downstream scaling behavior in terms of the fine-tuning size $n$.
Four representative meta-learning algorithms are adopted, which are discussed in \zcref{sec:related_work}: first- and second-order MAML~\autocite{finnModelAgnosticMetaLearningFast2017}, Prototypical Networks~\autocite{snell2017prototypical}, and R2-D2~\autocite{bertinetto2018metalearning}.
The meta-losses of these algorithms are augmented with a spectral norm-based regularizer on the model parameters.
Further experimental details are described in \zcref{app:experiments}.

\begin{figure}[t]
  \centering
  \includegraphics[width=\linewidth]{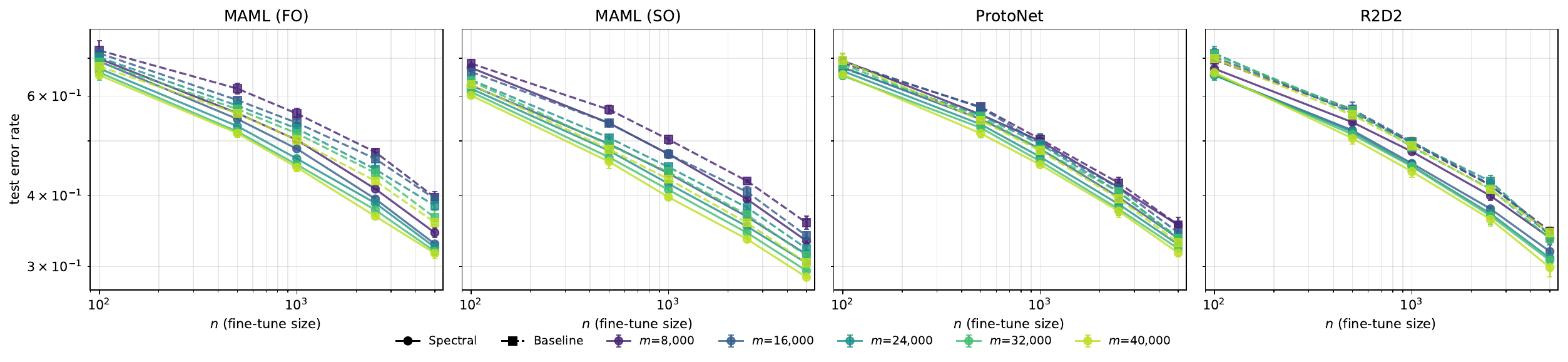}
  \caption{
    \textbf{Regularizing model complexity (parameter spectral norms) improves downstream sample efficiency.}
    Test error rate on CIFAR-10 (log scale) vs.\ fine-tuning size ($n$, log scale) with different pre-training size $m$ for four meta-learning algorithms trained on Mini-ImageNet in the 5-way 1-shot setting with and without model complexity regularization.
  }
  \label{fig:sample_efficiency_mini_spectral}
\end{figure}
\zcref{fig:sample_efficiency_mini_spectral} reports the downstream test error rates on CIFAR-10~\autocite{cifar2009} for a CNN, whose feature extractor is pre-trained with each meta-learning algorithm on Mini-ImageNet~\autocite{ravi2017optimization} with $m$ samples in the 5-way 1-shot setting and then fine-tuned on a subset of the CIFAR-10 training dataset with $n$ samples.
Regularizing the spectral norm of the parameters as a measure of model complexity yields a clear improvement in downstream performance.
Additional experiments in \zcref{app:experiments} show that complexity minimization works on other datasets and meta-learning settings, and regularizing other norms, such as the $\ell_1$ norm, also improves downstream sample efficiency.
Together, these algorithm-agnostic results provide empirical support for our theoretical claims.

\section{Related Work}\label{sec:related_work}
In this section, we briefly review prior works closely related to our study. 
A more comprehensive literature survey is provided in Appendix~\ref{sec:related_work_appendix}.

\paragraph{Meta-Learning Methodologies.}
Meta-learning aims to acquire a learning procedure that can rapidly adapt to unseen tasks from limited samples by exploiting experiences from a collection of past tasks drawn from a task distribution~\autocite{hospedales2021meta,wang2020generalizing}.
A standard taxonomy divides meta-learning methods into metric-based, optimization-based, and model-based approaches.
%
\textbf{Metric-based methods} classify queries by proximity, attention, or comparison in an embedding space.
Representative examples include Matching Networks, which introduced one-shot classification via attention over a support set and formalized episodic training~\autocite{vinyals2016matching}; Prototypical Networks, which classify queries by distances to class-wise mean embeddings and provide a clear view of meta-representation learning~\autocite{snell2017prototypical}; and Relation Networks, which learn the comparison function itself using a neural network~\autocite{sung2018learning}.
%
\textbf{Optimization-based methods} learn an initialization, update rule, or inner-loop adaptation mechanism such that a few optimization steps on a new task yield strong performance.
MAML established a model-agnostic framework based on inner-loop gradient descent and outer-loop optimization of post-adaptation performance~\autocite{finnModelAgnosticMetaLearningFast2017}.
%
Meta-SGD further learns the initialization, update directions, and learning rates~\autocite{li2017meta}.
R2-D2 replaces iterative inner-loop adaptation with a differentiable closed-form ridge-regression base learner on top of learned embeddings, occupying an intermediate position between metric-based classifiers and gradient-based adaptation methods~\autocite{bertinetto2018metalearning}.
%
%
\textbf{Model-based methods} implement adaptation within the network architecture itself, using memory, hypernetworks, or learned optimizers.
Memory-Augmented Neural Networks use external memory for rapid one/few-shot adaptation~\autocite{santoro2016meta}, while learning-to-learn approaches such as Optimization as a Model learn update rules using recurrent architectures~\autocite{ravi2017optimization,hochreiter2001learning,andrychowicz2016learning}.
SNAIL further combines temporal convolutions and attention as a general-purpose meta-learner across supervised and reinforcement learning domains~\autocite{mishra2017simple}.


\paragraph{Meta-Representation Learning: Sharing Representations Across Tasks}

Meta-learning is closely related to transfer learning, since both transfer information from previous tasks to unseen ones~\autocite{pan2009survey}.
Representation learning motivates the acquisition of shared latent features that facilitate learning across tasks~\autocite{bengio2013representation}, and meta-representation learning specializes such shared representations for few-shot task adaptation with statistical and computational efficiency.
Classically, Baxter's inductive bias learning model formalized meta-generalization as learning a good hypothesis space from multiple tasks sampled from a task environment~\autocite{baxter2000model}.
Subsequent work on Multi-Task Representation Learning established generalization bounds showing the benefit of learning low-dimensional dictionaries or feature maps shared across tasks~\autocite{maurer2016benefit,maurer2013sparse}.
More recent theory studies sample-efficient estimation and transfer of shared low-dimensional linear representations across linear regression tasks~\autocite{tripuraneni2021provable}, as well as the role of overparameterization in enabling few-shot adaptation with large-scale models~\autocite{sun2021towards}.


\paragraph{Learning Theory of (Deep) Meta-Learning}
Learning theory for meta-learning must handle a dual-sampling structure: tasks are sampled from a task distribution, and data points are sampled within each task~\autocite{baxter2000model,hospedales2021meta}.
Recent studies typically decompose excess risk into statistical estimation, optimization, and approximation errors, often through the meta-generalization gap between the expected risk on unseen tasks and the empirical meta-objective~\autocite{rezazadeh2022unified,wang2024stability}.
This line of work has clarified how representation sharing and the number of inner-loop adaptation steps affect sample efficiency and stability~\autocite{huang2022provable,chen2022understanding}.
Several theoretical frameworks have been developed.
Algorithmic stability, including meta-stability for both inner and outer loops, yields realistic bounds for gradient-based and non-convex meta-learning algorithms~\autocite{wang2024stability,bousquet2002stability}.
PAC-Bayes theory introduces hierarchical meta-priors and task-specific posteriors to obtain bounds depending on both the number of tasks and within-task sample size~\autocite{pentina2014pac,amit2018meta,rezazadeh2022unified}.
Information-theoretic analyses bound generalization via mutual information between algorithm outputs and data~\autocite{chen2021generalization}.
Uniform convergence remains a classical approach, but its bounds are often loose for deep learning and meta-learning, motivating the recent shift toward data-dependent analyses~\autocite{nagarajan2019uniform}.
Crucially, these results may not explain the data scaling law for pre-training, as their error rates with respect to the downstream sample size are independent of the pre-training sample size.


\section{Conclusion}\label{sec:conclusion}
We introduced complexity minimization, a meta representation learning framework that provably achieves the data scaling law for pre-training. The framework instructs each base-learner to estimate the best model complexity for its domain via Lepski's method, and the meta-learner selects the feature extractor that minimizes the worst-case complexity across all observed domains. Our end-to-end theoretical analysis establishes that the downstream excess error decays faster with the downstream sample size $n$ as the meta-training size $m$ grows, formally capturing the data scaling law. Empirically, augmenting standard meta-learning algorithms with a norm-based complexity regularizer consistently improves downstream sample efficiency across multiple algorithms and datasets, supporting the theoretical claims.

\paragraph{Limitations and broader impacts}
While the complexity assumptions on $\Phi$ in \zcref{asm:complex-phi} are reasonable, exhibiting a concrete class $\Phi$ that satisfies them remains open, and constructing such feature families is an important direction for future work. Our theoretical analysis likewise focuses on regression; extending it to classification and structured prediction remains open. We build on standard meta-learning and do not propose a qualitatively new paradigm with novel direct societal risks. As in broader work on large-scale pre-training and foundation models, familiar considerations regarding computational cost and equitable access to data and compute nonetheless apply.

\subsubsection*{Acknowledgments}
This work was partly supported by JSPS KAKENHI Grant Numbers JP26K02874 and JP23H00483 to K.F., JP23K28146, JP24K20836 and 25K03086 to K.M, and JST BOOST Grant Number JPMJBY24G2 to R.H.

\section*{References}
\printbibliography[heading=none]

\appendix

\section{Analyses}\label{sec:analyses}

\paragraph{Additional notations}
We fix a common probability space $(\Omega, \mathcal{F}, \mu)$ and identify all random variables with measurable maps on it. For a random variable $X \colon \Omega \to \RealSet$, we set $\Vab*{X}_p = (\Mean[|X|^p])^{1/p}$ for $p \in [1,\infty)$ and $\Vab*{X}_{\infty} = \inf\Bab*{C>0 : \p[|X|\le C]=1}$. For a measurable function $f\colon\mathcal{X}\to\RealSet$ and a random variable $X$ on $\mathcal{X}$, we set $\Vab*{f}_{L^\infty(X)} = \inf\Bab*{C>0:\p[|f(X)|\le C]=1}$. For a set $A$, $|A|$ denotes its cardinality; $\ind$ denotes the indicator function. For real values $a,b$, define $a \lor b = \max\Bab{a,b}$ and $a \land b = \min\Bab{a,b}$.

We write $\alpha_{\mathcal{P}}^* = \sup_{\phi \in \Phi}\inf_{(X,Y)\in\mathcal{P}}\alpha(\phi; X,Y)$, $\alpha_{\mathcal{P}}(\phi) = \inf_{(X,Y)\in\mathcal{P}}\alpha(\phi; X,Y)$, and $\beta_{\mathcal{P}}(\phi) = \inf_{(X,Y)\in\mathcal{P}}\beta(\phi; X,Y)$.

Given a function $h:\mathcal{X}\to\RealSet$ and $k \in \NaturalSet$, define the empirical $L^p(X)$ norm of $h$ for $p \in [1, \infty)$ and a random variable $X \in \mathcal{X}$ by
\begin{align}
  \Vab{h}_{L^p_k(X)}^p = \frac{1}{k}\sum_{i=1}^k h^p(X_i),
\end{align}
where $X_1, \ldots, X_k$ are i.i.d.\ copies of $X$. Given two random variables $X, Y \in \RealSet$ and $k \in \NaturalSet$, we use the empirical inner product and the empirical $L^p$ norm of $X$ for $p \in [1,\infty]$ defined as
\begin{align}
  \ab<X, Y>_{k} = \frac{1}{k}\sum_{i=1}^k X_i Y_i \mbox{ and } \Vab{X}_{L^p_k}^p = \frac{1}{k}\sum_{i=1}^kX_i^p,
\end{align}
where $(X_1, Y_1), \ldots, (X_k, Y_k)$ are i.i.d.\ copies of $(X, Y)$. Let
\begin{align}
  B_{J,J'}(\phi;X,Y) = \Vab{f^*_J\circ\phi - f^*_{J'}\circ\phi}^2_{L^2(X)} \mbox{ and } \hat{B}_{J,J'}(\phi;X,Y) = \Vab{\hat{f}_J\circ\phi - \hat{f}_{J'}\circ\phi}^2_{L^2_m(X)},
\end{align}
where we refer to these quantities as bias terms. Given a feature extractor $\phi$ and $k \in \NaturalSet$, the ideal and empirical best complexities are defined as
\begin{align}
  J^*_k(\phi) = \sup_{(X,Y) \in \mathcal{P}}J^*_k(\phi; X, Y) \mbox{ and } \hat{J}_k(\phi) = \max_{d \in [D]}\hat{J}_k(\phi; X^{(d)}, Y^{(d)}).
\end{align}
Here, $\hat{J}_k(\phi; X, Y)$ (two arguments) denotes the empirical Lepski complexity for a specific domain $(X,Y)$ and sample size $k$, as in \zcref{thm:complexity-bounds}, while $\hat{J}_k(\phi)$ (one argument) is its worst-case value over the observed domains. The ideal feature extractor $\phi^*$ is such that $J^*_m(\phi^*) = \inf_{\phi \in \Phi} J^*_m(\phi)$ and the estimated extractor $\hat\phi$ satisfies $\hat{J}_m(\hat\phi) = \min_{\phi \in \Phi}\hat{J}_m(\phi)$, consistent with \zcref{eq:meta-feature-extractor}. We define the empirical counterpart of $f^*_{Y|X,\phi,J}$ as
\begin{align}
  \hat{f}_{Y|X,\phi,k,J} = \argmin_{f_J \in \mathcal{F}_J}\frac{1}{k}\sum_{i=1}^k\ab((f_J\circ\phi)(X_i) - Y_i)^2,
\end{align}
where $(X_1,Y_1), \ldots, (X_k, Y_k)$ are i.i.d.\ copies of $(X, Y)$. We also omit the first and second subscripts in $\hat{f}_{Y|X,\phi,k,J}$ if $(X,Y)$ and $\phi$ are clear from the context.

\paragraph{Meta-learning and downstream learning analyses}
We analyze the proposed meta-learning and downstream learning algorithms and establish the performance guarantees. Specifically, we show the following two theorems corresponding to meta-learning and downstream learning, respectively.
\begin{theorem}[Meta-learning performance guarantee]\label{thm:meta-learning-analysis}
  Assume \zcref{asm:poly-approx,asm:complex-phi,asm:unif-domain,asm:sub-gauss-noise}. Suppose that $\ln^{1+\nu\gamma} m \lesssim D$ and $\ln D \lesssim \ln m$. For any $\mathcal{P} \in \mathfrak{P}$, the feature extractor $\hat\phi$ in \zcref{eq:meta-feature-extractor} satisfies that there exists a constant $C > 0$ such that with probability at least $1 - O(m^{-1})$,
  \begin{align}
    \beta_{\mathcal{P}}(\hat\phi) \ge \beta_{\mathcal{P}}^* - \frac{C}{\ln^\gamma m}.
  \end{align}
\end{theorem}
\begin{theorem}[Downstream learning performance guarantee]\label{thm:downstream-learning-analysis}
  Let $\mathcal{P} \in \mathfrak{P}$ be the target distribution. Suppose that the learned feature extractor $\hat\phi$ is independent of the downstream sample. Assume \zcref{asm:poly-approx}. Then, the learned regressor $\hat{f}_n$ in \zcref{sec:downstream-algorithm} satisfies that with probability at least $1 - n^{-1}$,
  \begin{align}
    \bar{E}_{\hat{\phi}}(\hat{f}_n) \lesssim (n/\ln n)^{-\beta_{\mathcal{P}}(\hat\phi)}.
  \end{align}
\end{theorem}
Combining \zcref{thm:meta-learning-analysis} and \zcref{thm:downstream-learning-analysis} immediately leads to the main theorem \zcref{thm:main}.

To prove \zcref{thm:meta-learning-analysis,thm:downstream-learning-analysis}, we first derive an error bound on the best complexity estimator $\hat{J}_m$ with a fixed feature extractor $\phi$ in \zcref{sec:best-complexity-est-fixed}. Then, we extend it to the learned feature extractor $\hat\phi$ in \zcref{sec:best-complexity-est-learned}. Building from these analyses, we prove \zcref{thm:meta-learning-analysis,thm:downstream-learning-analysis} in \zcref{sec:learning-analyses} by appropriately handling the effect of finite observation of domains.

\subsection{Best Complexity Estimation with Fixed Feature Extractor and Distribution}\label{sec:best-complexity-est-fixed}
In this subsection, we analyze the best complexity estimator $\hat{J}_m(\phi; X, Y)$ with a fixed feature extractor $\phi$ and a fixed distribution $(X, Y) \in \mathcal{P}$. Specifically, we prove the following theorem.
\begin{theorem}\label{thm:complexity-bounds}
  Fix $\phi$ and $(X, Y) \in \mathcal{P}$. Suppose that $A_J(\phi; X, Y) \asymp J^{-2\alpha}$ for some $\alpha > 0$. Assume \zcref{asm:sub-gauss-noise}. If $\rho > 0$ is a sufficiently large constant, there exist constants $c,C > 0$ such that with probability at least $1 - m^{-\Omega((m/\ln m)^{1/(2\alpha+1)})}$,
  \begin{align}
    c J^*_m(\phi; X, Y) \le \hat{J}_m(\phi; X, Y) \mbox{ and } \hat{J}_m(\phi; X, Y) \le C J^*_m(\phi; X, Y),
  \end{align}
  or equivalently, $\hat{J}_m(\phi; X, Y) \asymp (m/\ln m)^{1/(2\alpha+1)}$.
\end{theorem}
\zcref{thm:complexity-bounds} states that while the learner does not use the unknown parameter $\alpha$ in the complexity estimation, the resulting best complexity estimator $\hat{J}_m(\phi; X, Y)$ is equivalent to the ideal complexity depending on $\alpha$ up to multiplicative constants.

The key ingredient to prove \zcref{thm:complexity-bounds} is the upper and lower bounds on $\hat{B}_{J,J'}$ via the approximation error $A_J(\phi; X, Y)$.
\begin{theorem}\label{thm:bias-approx-error-bounds}
  Fix $\phi$, $(X, Y) \in \mathcal{P}$, and $J, J' \in [m]$ such that $J' \ge J$. Assume \zcref{asm:sub-gauss-noise}. Then, with probability at least $1 - e^{-t}$,
  \begin{align}
    \hat{B}_{J,J'}(\phi; X, Y) \lesssim A_J(\phi; X, Y) + V_{m, J'} + \frac{t}{m}.
  \end{align}
  Moreover, with probability at least $1 - e^{-t}$,
  \begin{align}
    \hat{B}_{J,J'}(\phi; X, Y) \gtrsim A_J(\phi; X, Y) - A_{J'}(\phi; X, Y) - V_{m, J'} - \frac{t}{m}.
  \end{align}
\end{theorem}
\zcref{thm:bias-approx-error-bounds} shows that the empirical bias term $\hat{B}_{J,J'}$ is characterized dominantly by the approximation error $A_J$ for a small complexity $J$.
In particular, when $A_J \asymp J^{-2\alpha}$, the empirical comparison $\hat{B}_{J,J'}$ between regressors at complexity levels $J$ and $J'$ faithfully reflects the underlying approximation structure: it is large when $A_J$ is large (meaning complexity $J$ is insufficient) and small when $A_J \asymp V_{m,J'}$ (meaning the Lepski stopping criterion is met). This fidelity is what enables the Lepski method to identify the optimal complexity from data without prior knowledge of $\alpha$.

We also use the following property of the approximation error.
\begin{lemma}\label{lem:population-lepski-condition}
  For fixed $\phi$ and $(X, Y) \in \mathcal{P}$, assume $A_J(\phi; X, Y) \asymp J^{-2\alpha}$ for some $\alpha \in (0,\infty)$. Then, we have
  \begin{align}
    A_{J^*} \asymp V_{m, J^*} \asymp (m/\ln m)^{-2\alpha/(2\alpha+1)} \mbox{ and } J^* \asymp (m/\ln m)^{1/(2\alpha+1)},
  \end{align}
  where $J^* = J^*_m(\phi; X, Y)$.
\end{lemma}
\zcref{lem:population-lepski-condition} states that $J^*$ achieves the ideal complexity of $(m/\ln m)^{1/(2\alpha+1)}$, and $A_{J^*}$ and $V_{m, J^*}$ are equivalent up to multiplicative constants.

Now, we prove \zcref{thm:complexity-bounds}.
\begin{proof}[Proof of \zcref{thm:complexity-bounds}]
  Let $J^* = J^*_m(\phi; X, Y)$, $\hat{J} = \hat{J}_m(\phi; X, Y)$, $\hat{B}_{J,J'} = \hat{B}_{J,J'}(\phi; X, Y)$, and $A_J = A_J(\phi; X, Y)$. Applying the union bound over $J,J' \in [m]$ into \zcref{thm:bias-approx-error-bounds} gives that with probability at least $1-e^{-t}$, for all $J,J' \in [m]$ where $J' \ge J$,
  \begin{align}
    \hat{B}_{J,J'} \lesssim A_J + V_{m,J'} + \frac{t}{m}, \label{eq:complexity-bounds-1-upper}
  \end{align}
  and
  \begin{align}
    \hat{B}_{J,J'} \gtrsim A_J - A_{J'} - V_{m,J'} - \frac{t}{m}, \label{eq:complexity-bounds-1-lower}
  \end{align}
  where we use $\frac{\ln m}{m} \lesssim V_{m,J} \lesssim V_{m,J'}$. Let $\mathcal{E}_t$ be this event; hence, $\p\Bab*{\mathcal{E}_t} \ge 1 - e^{-t}$.

  \paragraph*{Upper bound.} Let $J_{\mathrm{up}} \in [m]$ such that $CJ^* \ge J_{\mathrm{up}} \ge CJ^*-1$ for some $C > 1$. Suppose that $\mathcal{E}_t$ occurs. Assume that $\hat{J} > J_{\mathrm{up}}$. Then, $J_{\mathrm{up}}$ must fail the empirical Lepski condition; hence, for some $J' \ge J_{\mathrm{up}}$, we have
  \begin{align}
    \hat{B}_{J_{\mathrm{up}},J'} > \rho V_{m,J'}. \label{eq:complexity-bounds-upper-bound-1}
  \end{align}
  By \zcref{eq:complexity-bounds-1-upper}, there exist constants $C_1, C_2 > 0$ such that
  \begin{align}
    \rho V_{m,J'} < C_1\ab(A_{J_{\mathrm{up}}} + V_{m,J'}) + C_2\frac{t}{m}.
  \end{align}
  By \zcref{lem:population-lepski-condition}, we have
  \begin{align}
    A_{J_{\mathrm{up}}} \le A_{J^*} \lesssim V_{m,J^*} \le V_{m,J'}.
  \end{align}
  Hence, there are constants $C_1, C_2 > 0$ such that
  \begin{align}
    \frac{C_2t}{m} > \ab(\rho - C_1)V_{m,J'},
  \end{align}
  where the sufficiently large $\rho$ ensures $\rho - C_1 > 0$. The above inequality is contradicting if
  \begin{align}
    t \le \frac{m \ab(\rho - C_1)V_{m,J'}}{C_2}.
  \end{align}
  For such $t$, $\hat{J} \le J_{\mathrm{up}}$ under the event $\mathcal{E}_t$. Choosing $t = \ln m \cdot \kappa (m/\ln m)V_{m,J^*}$ for $\kappa = (\rho - C_1)/C_2$ yields the contradictory $t$. Consequently, with probability at least $\p\Bab{\mathcal{E}_t} \ge 1 - m^{-\kappa (m/\ln m)V_{m,J^*}}$, we have $\hat{J} \le J_{\mathrm{up}} \le C J^*$.

  \paragraph*{Lower bound.} Let $J_{\mathrm{lo}} \in [m]$ such that $cJ^* + 1 \ge J_{\mathrm{lo}} \ge cJ^*$ for some $c \in (0,1)$. Suppose that $\mathcal{E}_t$ occurs. Assume that $\hat{J} < J_{\mathrm{lo}}$. Then, since $\hat{J}$ satisfies the empirical Lepski condition, we have
  \begin{align}
    \hat{B}_{\hat{J}, J^*} \le \rho V_{m, J^*}.
  \end{align}
  By \zcref{eq:complexity-bounds-1-lower}, there exist constants $C_1, C_2, C_3 > 0$ such that
  \begin{align}
    \rho V_{m, J^*} \ge C_1A_{\hat{J}} - C_2\ab(A_{J^*} + V_{m, J^*}) - \frac{C_3t}{m}.
  \end{align}
  By the assumption of $A_J \asymp J^{-2\alpha}$ and \zcref{lem:population-lepski-condition}, we have
  \begin{align}
    A_{\hat{J}} \gtrsim c^{-2\alpha}A_{J^*} \gtrsim c^{-2\alpha}V_{m, J^*},
  \end{align}
  and $A_{J^*} \lesssim V_{m,J^*}$. Hence, there exist constants $C_1, C_2, C_3 > 0$ such that
  \begin{align}
    \frac{C_3t}{m} \ge \ab(C_1c^{-2\alpha} - C_2 - \rho)V_{m, J^*}.
  \end{align}
  A sufficiently small $c$ ensures $C_1c^{-2\alpha} - C_2 - \rho > 0$. The above inequality is contradicting if
  \begin{align}
    t \le \frac{m \ab(C_1c^{-2\alpha} - C_2 - \rho)V_{m, J^*}}{C_3}.
  \end{align}
  For such $t$, $\hat{J} \ge J_{\mathrm{lo}}$ under the event $\mathcal{E}_t$. Choosing $t = \ln m \cdot \kappa (m/\ln m)V_{m,J^*}$ for $\kappa = (C_1c^{-2\alpha} - C_2 - \rho)/C_3$ yields the contradictory $t$. Consequently, with probability at least $\p\Bab{\mathcal{E}_t} \ge 1 - m^{-\kappa (m/\ln m)V_{m,J^*}}$, we have $\hat{J} \ge J_{\mathrm{lo}} \ge cJ^*$.

  By \zcref{lem:population-lepski-condition}, we have $(m/\ln m)V_{m,J^*} \gtrsim (m/\ln m)^{1/(2\alpha + 1)}$, which gives the claim.
\end{proof}

\subsection{Bias Analysis with Fixed Feature Extractor}\label{sec:bias-fixed}
In this subsection, we investigate the empirical bias term $\hat{B}_{J,J'}$ to prove \zcref{thm:bias-approx-error-bounds}. To derive bounds on $\hat{B}_{J,J'}$, we first derive the error upper and lower bounds on the sieved least-square estimator $\hat{f}_{m,J}$. 
\paragraph{Sieved least-square estimator}
\begin{lemma}\label{lem:sieved-LS-oracle}
  Fix $\phi$, $J$, and $(X, Y) \in \mathcal{P}$. Assume \zcref{asm:sub-gauss-noise}. Then, there exists universal constants $c_1, c_2 > 0$ such that for any $\eta \in (0,1)$, with probability at least $1 - 3e^{-t}$,
  \begin{multline}
    \Vab{\hat{f}_J \circ \phi - \Mean[Y | X]}_{L^2_{m}(X)}^2 \le \\ \frac{1}{1-\eta}\pab[Bigg]{2(1+\eta)A_J(\phi; X, Y) + \frac{6\sigma^2 c_2J\ln m}{\eta m} \\ + \ab(2\sqrt{2\ln(2)}\sigma c_1 + \frac{\eta c_1^2}{m})\frac{1}{m} + \ab(\frac{2\sqrt{2}\sigma c_1}{\ln^{1/2}(2)}+\frac{6\sigma^2}{\eta}+\frac{5(1+\eta)}{6})\frac{t}{m}}.
  \end{multline}
\end{lemma}
\begin{lemma}\label{lem:sieved-LS-lower}
  Fix $\phi$, $J$, and $(X, Y) \in \mathcal{P}$. Assume \zcref{asm:sub-gauss-noise}. Then, there exists universal constants $c_1, c_2 > 0$ such that with probability at least $1 - e^{-t}$,
  \begin{align}
    \Vab{\hat{f}_J \circ \phi - \Mean[Y | X]}_{L^2_{m}(X)}^2 \ge \frac{1}{4}A_J(\phi; X, Y) - \frac{2c_2J\ln m}{3m} - \frac{3c_1^2}{2m^2} - \frac{2t}{3m}.
  \end{align}
\end{lemma}
\begin{remark}
  \zcref{lem:sieved-LS-oracle} implies that, with probability at least $1 - e^{-t}$,
  \begin{align}
    \Vab{\hat{f}_J \circ \phi - \Mean[Y | X]}_{L^2_{m}(X)}^2 \lesssim A_J(\phi; X, Y) + V_{m,J} + \frac{t}{m}.
  \end{align}
  Also, \zcref{lem:sieved-LS-lower} implies that, with probability at least $1 - e^{-t}$,
  \begin{align}
    \Vab{\hat{f}_J \circ \phi - \Mean[Y | X]}_{L^2_{m}(X)}^2 \gtrsim A_J(\phi; X, Y) - V_{m,J} - \frac{t}{m}.
  \end{align}
\end{remark}

Now, we prove \zcref{thm:bias-approx-error-bounds}.
\begin{proof}[Proof of \zcref{thm:bias-approx-error-bounds}]
  For shorthands, let $\hat{B}_{J,J'} = \hat{B}_{J,J'}(\phi; X, Y)$ and $A_J = A_J(\phi; X, Y)$. By the triangle and reverse triangle inequalities, we have
  \begin{multline}
    \ab(\Vab{\hat{f}_{m,J}\circ\phi - \Mean[Y|X]}_{L^2_m(X)} - \Vab{\hat{f}_{m,J'}\circ\phi - \Mean[Y|X]}_{L^2_m(X)})^2 \\ \le \hat{B}_{J,J'} \le \\ \ab(\Vab{\hat{f}_{m,J}\circ\phi - \Mean[Y|X]}_{L^2_m(X)} + \Vab{\hat{f}_{m,J'}\circ\phi - \Mean[Y|X]}_{L^2_m(X)})^2. \label{eq:bias-approx-error-bounds-1}
  \end{multline}
  We now prove the upper and lower bounds separately.

  \paragraph{Upper bound}
  Application of \zcref{lem:sieved-LS-oracle} to both terms in the right-hand side of \zcref{eq:bias-approx-error-bounds-1} yields that with probability at least $1-2e^{-t}$,
  \begin{align}
    \Vab{\hat{f}_{m,J}\circ\phi - \Mean[Y|X]}_{L^2_m(X)}^2 \lesssim& A_J + V_{m,J} + \frac{t}{m} \\
    \Vab{\hat{f}_{m,J'}\circ\phi - \Mean[Y|X]}_{L^2_m(X)}^2 \lesssim& A_{J'} + V_{m,J'} + \frac{t}{m}.
  \end{align}
  Noting that $A_J$ is decreasing in $J$, $V_{m,J}$ is increasing in $J$, and $V_{m,J} \gtrsim \frac{1}{m}$, we have with probability at least $1-e^{-t}$,
  \begin{align}
    \hat{B}_{J,J'} \lesssim A_J + V_{m,J'} + \frac{t}{m}.
  \end{align}

  \paragraph{Lower bound}
  Respectively applying \zcref{lem:sieved-LS-lower} and \zcref{lem:sieved-LS-oracle} into the first and second terms of the left-hand side of \zcref{eq:bias-approx-error-bounds-1} gives that with probability at least $1-2e^{-t}$,
  \begin{align}
    \Vab{\hat{f}_{m,J}\circ\phi - \Mean[Y|X]}_{L^2_m(X)}^2 \gtrsim& A_J - V_{m,J} - \frac{t}{m} \\
    \Vab{\hat{f}_{m,J'}\circ\phi - \Mean[Y|X]}_{L^2_m(X)}^2 \lesssim& A_{J'} + V_{m,J'} + \frac{t}{m}.
  \end{align}
  Again, using the facts that $A_J$ is decreasing in $J$, $V_{m,J}$ is increasing in $J$, and $V_{m,J} \gtrsim \frac{1}{m}$, we have with probability at least $1-e^{-t}$,
  \begin{align}
    \hat{B}_{J,J'} \gtrsim A_J - A_{J'} - V_{m,J'} - \frac{t}{m}.
  \end{align}
\end{proof}

\subsection{Best Complexity Estimator with Learned Feature Extractor}\label{sec:best-complexity-est-learned}
In this subsection, we prove \zcref{thm:meta-learning-analysis}. To this end, we extend \zcref{thm:complexity-bounds} for the learned feature extractor $\hat\phi$. Specifically, we prove the following theorem.
\begin{theorem}\label{thm:complexity-bounds-learned}
  Fix $(X, Y) \in \mathcal{P}$. Let $\hat\phi$ be the learned feature extractor depending on the pre-training samples. Assume \zcref{asm:complex-phi,asm:poly-approx}. If $\rho > 0$ is a sufficiently large constant, there exist constants $c,C > 0$ such that for some $\epsilon > 0$,  with probability at least $1 - \sum_{J,J'\in[m]:J' \ge J}N(V_{m,J}, \Phi, \rho_J)N(V_{m,J'}, \Phi, \rho_{J'})m^{-\Omega((m/\ln m)^{1/(2\bar\alpha+1)})}$,
  \begin{align}
    c J^*_m(\hat\phi; X, Y) \le \hat{J}_m(\hat\phi; X, Y) \mbox{ and } \hat{J}_m(\hat\phi; X, Y) \le CJ^*_m(\hat\phi; X, Y),
  \end{align}
  or equivalently, $\hat{J}_m(\hat\phi; X, Y) \asymp (m/\ln m)^{1/(2\alpha(\hat\phi; X, Y)+1)}$.
\end{theorem}
\zcref{thm:complexity-bounds-learned} is an analogue of \zcref{thm:complexity-bounds} with learned $\hat\phi$.

Following the analyses with the fixed $\phi$ case, we derive bounds on the bias terms with $\hat\phi$ to prove \zcref{thm:complexity-bounds-learned}.
\begin{theorem}\label{thm:bias-approx-error-bounds-learned}
  Fix $\phi$, $(X, Y) \in \mathcal{P}$, and $J, J' \in [m]$ such that $J' \ge J$. Let $\hat\phi$ be the learned feature extractor depending on the pre-training samples. Assume \zcref{asm:complex-phi,asm:sub-gauss-noise}. Then, with probability at least $1 - N(V_{m,J}, \Phi, \rho_J)N(V_{m,J'}, \Phi, \rho_{J'})e^{-t}$,
  \begin{align}
    \hat{B}_{J,J'}(\hat\phi; X, Y) \lesssim A_J(\hat\phi; X, Y) + V_{m, J'} + \frac{t}{m}.
  \end{align}
  Moreover, with probability at least $1 - N(V_{m,J}, \Phi, \rho_J)N(V_{m,J'}, \Phi, \rho_{J'})e^{-t}$,
  \begin{align}
    \hat{B}_{J,J'}(\hat\phi; X, Y) \gtrsim A_J(\hat\phi; X, Y) - A_{J'}(\hat\phi; X, Y) - V_{m, J'} - \frac{t}{m}.
  \end{align}
\end{theorem}

Based on \zcref{thm:bias-approx-error-bounds-learned}, we prove \zcref{thm:complexity-bounds-learned}.
\begin{proof}[Proof of \zcref{thm:complexity-bounds-learned}]
  We invoke the same proof of \zcref{thm:complexity-bounds} except leveraging \zcref{thm:bias-approx-error-bounds-learned} and choosing $t$ as
  \begin{align}
    t = \ln m \cdot \kappa (m/\ln m)^{1/(2\bar\alpha + 1)},
  \end{align}
  where $\kappa$ is an appropriate constant leading to the contradictory $t$. Noting that $m/\ln m > 1$, such a choice of $t$ gives contradictory $t$ for any $\hat\phi$ and $(X, Y) \in \mathcal{P}^*$.
\end{proof}

\subsection{Meta-Learning and Downstream Learning Analyses}\label{sec:learning-analyses}
Now, we prove \zcref{thm:meta-learning-analysis}.
\begin{proof}[Proof of \zcref{thm:meta-learning-analysis}]
  By the definition of $\hat\phi$, we have
  \begin{align}
    \max_{d \in [D]}\hat{J}_m(\hat\phi; X^{(d)}, Y^{(d)}) \le \max_{d \in [D]}\hat{J}_m(\phi^*; X^{(d)}, Y^{(d)}).
  \end{align}

  Application of the union bound into \zcref{thm:complexity-bounds-learned} over $d \in [D]$ yields that with probability at least $1 - D\sum_{J,J'\in[m]:J' \ge J}N(V_{m,J}, \Phi, \rho_J)N(V_{m,J'}, \Phi, \rho_{J'})m^{-\Omega((m/\ln m)^{1/(2\bar\alpha+1)})}$,
  \begin{align}
    \max_{d \in [D]}(m/\ln m)^{1/(2\alpha(\hat\phi; X^{(d)}, Y^{(d)}) + 1)} \lesssim \max_{d \in [D]}J^*_m(\hat\phi; X^{(d)}, Y^{(d)}) \lesssim \max_{d \in [D]}\hat{J}_m(\hat\phi; X^{(d)}, Y^{(d)})
  \end{align}

  Also, applying the union bound to \zcref{thm:complexity-bounds} with $\phi = \phi^*$ over $d \in [D]$ gives that with probability at least $1 - Dm^{-\Omega((m/\ln m)^{1/(2\bar\alpha + 1)})}$,
  \begin{align}
    \max_{d \in [D]}\hat{J}_m(\phi^*; X^{(d)}, Y^{(d)}) \lesssim \max_{d \in [D]}J^*_m(\phi^*; X^{(d)}, Y^{(d)}) \le J^*_m(\phi^*) \lesssim (m/\ln m)^{1/(2\alpha^*_{\mathcal{P}} + 1)}.
  \end{align}

  Note that $1/(2\alpha(\phi; X, Y) + 1) = 1 - \beta(\phi; X, Y)$. Consequently, there exists a constant $C > 1$ such that with probability at least $1 - 2D\sum_{J,J'\in[m]:J' \ge J}N(V_{m,J}, \Phi, \rho_J)N(V_{m,J'}, \Phi, \rho_{J'})m^{-\Omega((m/\ln m)^{1/(2\bar\alpha+1)})}$,
  \begin{align}
    (m/\ln m)^{1 - \min_{d \in [D]}\beta(\hat\phi; X^{(d)}, Y^{(d)})} \le C(m/\ln m)^{1 - \beta^*_{\mathcal{P}}}.
  \end{align}
  Taking the logarithm of both sides and dividing by $\ln(m/\ln m)$, we have
  \begin{align}
    \min_{d \in [D]}\beta(\hat\phi; X^{(d)}, Y^{(d)}) \ge \beta^*_{\mathcal{P}} - \frac{\ln C}{\ln(m/\ln m)}.
  \end{align}
  Fix an arbitrary $\phi \in \Phi$.   Let $(X^{[1]},Y^{[1]}),\ldots,(X^{[\lceil 1/\epsilon \rceil]}, Y^{[\lceil 1/\epsilon \rceil]})$ be such that for any $(X, Y) \in \mathcal{P}$, there exists $i$ satisfying $|\beta(\phi; X^{[i]}, Y^{[i]}) - \beta(\phi; X, Y)| \le \epsilon$. By \zcref{asm:unif-domain}, there exists $d \in [D]$ such that $|\min_i\beta(\phi; X^{[i]}, Y^{(i)}) - \beta(\phi; X^{(d)}, Y^{(d)})| \le \epsilon$ with probability at least $1 - (1-C\epsilon^{\nu})^D$ for some universal constant $C > 0$. Hence, with probability at least $1 - (1-C\epsilon^{\nu})^D$,
  \begin{align}
    \min_{d \in [D]}\beta(\phi; X^{(d)}, Y^{(d)}) \le \inf_{(X, Y) \in \mathcal{P}}\beta(\phi; X, Y) + 2\epsilon.
  \end{align}
  For $t > 0$, taking $\epsilon^{\nu} = t/CD$ yields with probability at least $1 - e^{-t}$,
  \begin{align}
    \min_{d \in [D]}\beta(\phi; X^{(d)}, Y^{(d)}) \le \inf_{(X, Y) \in \mathcal{P}}\beta(\phi; X, Y) + 2\ab(\frac{t}{CD})^\nu.
  \end{align}

  Consider $\ln^{-\gamma}m$-cover of $\Phi$ in \zcref{asm:complex-phi}, denoted as $\phi_1,\ldots,\phi_{N_m}$, where $N_m = N(\ln^{-\gamma}m, \Phi, \rho_{\beta,\mathcal{P}})$. Let $\hat\phi_m$ be the closest $\phi_i$ to $\hat\phi$ in terms of $\rho_{\beta,\mathcal{P}}$. Then, we have with probability at least $1 - N_me^{-t}$,
  \begin{align}
    \min_{d \in [D]}\beta(\hat\phi; X^{(d)}, Y^{(d)}) \le& \min_{d \in [D]}\beta(\hat\phi_m; X^{(d)}, Y^{(d)}) + \frac{1}{\ln^\gamma m} \\
    \le& \inf_{(X, Y) \in \mathcal{P}}\beta(\hat\phi_m; X, Y) + 2\ab(\frac{t}{CD})^\nu + \frac{1}{\ln^\gamma m} \\
    \le& \beta_{\mathcal{P}}(\hat\phi) + 2\ab(\frac{t}{CD})^\nu + \frac{2}{\ln^\gamma m}.
  \end{align}
  By \zcref{asm:complex-phi} and $D \gtrsim \ln^{1+\nu\gamma} m$, with $t = \ln N_m + \ln m$, we have with probability at least $1 - m^{-1}$,
  \begin{align}
    \min_{d \in [D]}\beta(\hat\phi; X^{(d)}, Y^{(d)}) \le \beta_{\mathcal{P}}(\hat\phi) + O\ab(\frac{1}{\ln^\gamma m}).
  \end{align}
  By the union bound, we have with probability at least $1 - m^{-1} - 2m^{O(1)}\sum_{J,J'\in[m]:J' \ge J}N(V_{m,J}, \Phi, \rho_J)N(V_{m,J'}, \Phi, \rho_{J'})m^{-\Omega((m/\ln m)^{1/(2\bar\alpha+1)})}$,
  \begin{align}
    \beta_{\mathcal{P}}(\hat\phi) \ge \beta^*_{\mathcal{P}} - O\ab(\frac{1}{\ln^\gamma m}).
  \end{align}
  By \zcref{asm:complex-phi}, we have
  \begin{multline}
    \sum_{J,J'\in[m]:J' \ge J}N(V_{m,J}, \Phi, \rho_J)N(V_{m,J'}, \Phi, \rho_{J'})m^{-\Omega((m/\ln m)^{1/(2\bar\alpha+1)})} \\ \lesssim m^{-\Omega((m/\ln m)^{1/(2\bar\alpha+1)})} \lesssim m^{-1},
  \end{multline}
  which gives the desired result.
\end{proof}

Next, we prove \zcref{thm:downstream-learning-analysis}.
\begin{proof}[Proof of \zcref{thm:downstream-learning-analysis}]
  Let $J^* = J^*_n(\hat\phi; X, Y)$ and $\hat{J} = \hat{J}_n(\hat\phi; X, Y)$. Let $\hat{f}'_{n,J}$ be the closest function among $O(1/n)$-net of $\mathcal{F}_J$ to $\hat{f}_{n,J}$ in $L^\infty$-norm. Application of the union bound over $O(1/n)$-net and $\hat{J}$ and \zcref{lem:squared-error-difference-concentration} yields that with probability at least $1 - e^{-t}$,
  \begin{align}
    \bar{E}_{\hat\phi}(\hat{f}_{n}) =& \Vab{\hat{f}_{n,\hat{J}}\circ\hat\phi - \Mean[Y|X]}_{L^2(X)}^2 \\
    \lesssim& \Vab{\hat{f}'_{n,\hat{J}}\circ\hat\phi - \Mean[Y|X]}_{L^2(X)}^2 + \frac{1}{n^2} \\
    \lesssim&
    \begin{multlined}[t]
      \Vab{\hat{f}'_{n,\hat{J}}\circ\hat\phi - \Mean[Y|X]}_{L^2_n(X)}^2 \\ + \sqrt{V_{n,\hat{J}} + \frac{\ln n}{n} + \frac{t}{n}}\Vab{\hat{f}'_{n,\hat{J}}\circ\hat\phi - \Mean[Y|X]}_{L^2(X)} + V_{n,\hat{J}} + \frac{\ln n}{n} + \frac{t}{n}.
    \end{multlined}
  \end{align}
  By AM-GM inequality, with probability at least $1 - e^{-t}$,
  \begin{align}
    \bar{E}_{\hat\phi}(\hat{f}_{n}) \lesssim \Vab{\hat{f}_{n,\hat{J}}\circ\hat\phi - \Mean[Y|X]}_{L^2_n(X)}^2 + V_{n,\hat{J}} + \frac{t}{n}.
  \end{align}

  If $\hat{J} \le J^*$, the empirical Lepski's rule ensures that
  \begin{align}
    \bar{E}_{\hat\phi}(\hat{f}_{n}) \lesssim \Vab{\hat{f}_{n,J^*}\circ\hat\phi - \Mean[Y|X]}_{L^2_n(X)}^2 + V_{n,J^*} + \frac{t}{n}.
  \end{align}
  By \zcref{lem:sieved-LS-oracle}, we get with probability at least $1 - e^{-t}$,
  \begin{align}
    \bar{E}_{\hat\phi}(\hat{f}_{n}) \lesssim A_{J^*}(\hat\phi; X, Y) + V_{n,J^*} + \frac{t}{n}.
  \end{align}
  From \zcref{lem:population-lepski-condition}, we have with probability at least $1 - e^{-t}$,
  \begin{align}
    \bar{E}_{\hat\phi}(\hat{f}_{n}) \lesssim (n/\ln n)^{-\beta(\hat\phi; X, Y)} + \frac{t}{n}.
  \end{align}

  If $\hat{J} > J^*$, there exists $J' \ge J^*$ such that $J^*$ fails the empirical Lepski condition, i.e.,
  \begin{align}
    \hat{B}_{J^*,J'}(\hat\phi; X, Y) > \rho V_{n,J'}.
  \end{align}
  By \zcref{thm:bias-approx-error-bounds} (upper bound), with probability at least $1-e^{-t}$,
  \begin{align}
    \hat{B}_{J^*,J'}(\hat\phi; X, Y) \lesssim A_{J^*}(\hat\phi; X, Y) + V_{n,J'} + \frac{t}{n}.
  \end{align}
  From \zcref{lem:population-lepski-condition}, $A_{J^*} \lesssim V_{n,J^*} \le V_{n,J'}$, so there exist constants $C_1,C_2 > 0$ such that
  \begin{align}
    \rho V_{n,J'} < C_1V_{n,J'} + C_2\frac{t}{n}.
  \end{align}
  Rearranging gives $(\rho - C_1)V_{n,J'} < C_2 t/n$. Choosing $\rho > C_1 + C_2$ and $t = \ln n$, this is contradicted since $V_{n,J'} \ge V_{n,1} = \ln(n)/n$ implies $(\rho-C_1)\ln(n)/n \le C_2\ln(n)/n$. Hence, by choosing $\rho$ sufficiently large, the event $\hat{J} > J^*$ occurs with probability at most $e^{-t} = 1/n$. Combining both cases with a union bound, we have with probability at least $1-2/n$,
  \begin{align}
    \bar{E}_{\hat\phi}(\hat{f}_{n}) \lesssim (n/\ln n)^{-\beta_{\mathcal{P}}(\hat\phi)} + \frac{\ln n}{n} \lesssim (n/\ln n)^{-\beta_{\mathcal{P}}(\hat\phi)},
  \end{align}
  where the last step uses $\beta_{\mathcal{P}}(\hat\phi) \in (0,1)$.
\end{proof}

\subsection{Bias Analysis with Learned Feature Extractor}
Here, we analyze $\hat{B}_{J,J'}$ for the learned feature extractor $\hat\phi$. We follow similar steps of the analyses in \zcref{sec:bias-fixed} but introduce an approximation and the union bound due to the covering of $\Phi$. First, we reveal the error upper and lower bounds on the sieved least-square estimator.
\begin{lemma}\label{lem:sieved-LS-oracle-learned}
  Fix $(X, Y) \in \mathcal{P}$, $J \in [m]$, and $\epsilon > 0$. Let $\hat\phi$ be the learned feature extractor depending on the pre-training samples. Assume \zcref{asm:complex-phi,asm:sub-gauss-noise}. Then, there exist universal constants $c_1,c_2 > 0$ such that for any $\eta \in (0,1)$, with probability at least $1 - (1+2N(\epsilon, \Phi, \rho_J))e^{-t}$,
  \begin{multline}
    \Vab{\hat{f}_J \circ \hat\phi - \Mean[Y | X]}_{L^2_{m}(X)}^2 \le \\ \frac{1}{1-\eta}\pab[Bigg]{4(2+3\eta)A_J(\hat\phi; X, Y) + \frac{6\sigma^2 c_2J\ln m}{\eta m} \\ + \ab(2\sqrt{2\ln(2)}\sigma c_1 + \frac{3\eta c_1^2}{m})\frac{1}{m} + (4+3\eta)\epsilon + \ab(\frac{2\sqrt{2}\sigma c_1}{\ln^{1/2}(2)}+\frac{6\sigma^2}{\eta}+\frac{5(1+\eta)}{6})\frac{t}{m}}.
  \end{multline}
\end{lemma}
\begin{lemma}\label{lem:sieved-LS-lower-learned}
  Fix $(X, Y) \in \mathcal{P}$, $J \in [m]$, and $\epsilon > 0$. Let $\hat\phi$ be the learned feature extractor depending on the pre-training samples. Assume \zcref{asm:complex-phi,asm:poly-approx}. Then, there exist universal constants $c_1,c_2 > 0$ such that with probability at least $1 - N(\epsilon, \Phi, \rho_J)e^{-t}$,
  \begin{align}
    \Vab{\hat{f}_J\circ\hat\phi - \Mean[Y|X]}_{L^2_m(X)}^2 \ge \frac{1}{4}A_J(\hat\phi; X, Y) - \frac{2c_2J\ln m}{3m} - \frac{3\epsilon^2}{2} - \frac{3c_1^2}{2m^2} - \frac{2t}{3m}.
  \end{align}
\end{lemma}

We now give a proof of \zcref{thm:bias-approx-error-bounds-learned}.
\begin{proof}[Proof of \zcref{thm:bias-approx-error-bounds-learned}]
  We follow the proof of \zcref{thm:bias-approx-error-bounds} with \zcref{lem:sieved-LS-oracle-learned,lem:sieved-LS-lower-learned}. For shorthands, let $\hat{B}_{J,J'} = \hat{B}_{J,J'}(\hat\phi; X, Y)$ and $A_J = A_J(\hat\phi; X, Y)$. By the triangle and reverse triangle inequalities, we have
  \begin{multline}
    \ab(\Vab{\hat{f}_{m,J}\circ\hat\phi - \Mean[Y|X]}_{L^2_m(X)} - \Vab{\hat{f}_{m,J'}\circ\hat\phi - \Mean[Y|X]}_{L^2_m(X)})^2 \\ \le \hat{B}_{J,J'} \le \\ \ab(\Vab{\hat{f}_{m,J}\circ\hat\phi - \Mean[Y|X]}_{L^2_m(X)} + \Vab{\hat{f}_{m,J'}\circ\hat\phi - \Mean[Y|X]}_{L^2_m(X)})^2. \label{eq:bias-approx-error-bounds-1-learned}
  \end{multline}
  We now prove the upper and lower bounds separately.

  \paragraph{Upper bound}
  Application of \zcref{lem:sieved-LS-oracle-learned} to both terms in the right-hand side of \zcref{eq:bias-approx-error-bounds-1-learned} yields that with probability at least $1-N(V_{m,J}, \Phi, \rho_J)N(V_{m,J'}, \Phi, \rho_{J'})e^{-t}$,
  \begin{align}
    \Vab{\hat{f}_{m,J}\circ\hat\phi - \Mean[Y|X]}_{L^2_m(X)}^2 \lesssim& A_J + V_{m,J} + \frac{t}{m} \\
    \Vab{\hat{f}_{m,J'}\circ\hat\phi - \Mean[Y|X]}_{L^2_m(X)}^2 \lesssim& A_{J'} + V_{m,J'} + \frac{t}{m}.
  \end{align}
  Noting that $A_J$ is decreasing in $J$, $V_{m,J}$ is increasing in $J$, and $V_{m,J} \gtrsim \frac{1}{m}$, we have with probability at least $1-N(V_{m,J}, \Phi, \rho_J)N(V_{m,J'}, \Phi, \rho_{J'})e^{-t}$,
  \begin{align}
    \hat{B}_{J,J'} \lesssim A_J + V_{m,J'} + \frac{t}{m}.
  \end{align}

  \paragraph{Lower bound}
  Respectively applying \zcref{lem:sieved-LS-lower-learned} and \zcref{lem:sieved-LS-oracle-learned} into the first and second terms of the left-hand side of \zcref{eq:bias-approx-error-bounds-1-learned} gives that with probability at least $1-N(V_{m,J}, \Phi, \rho_J)N(V_{m,J'}, \Phi, \rho_{J'})e^{-t}$,
  \begin{align}
    \Vab{\hat{f}_{m,J}\circ\hat\phi - \Mean[Y|X]}_{L^2_m(X)}^2 \gtrsim& A_J - V_{m,J} - \frac{t}{m} \\
    \Vab{\hat{f}_{m,J'}\circ\hat\phi - \Mean[Y|X]}_{L^2_m(X)}^2 \lesssim& A_{J'} + V_{m,J'} + \frac{t}{m}.
  \end{align}
  Again, using the facts that $A_J$ is decreasing in $J$, $V_{m,J}$ is increasing in $J$, and $V_{m,J} \gtrsim \frac{1}{m}$, we have with probability at least $1-N(V_{m,J}, \Phi, \rho_J)N(V_{m,J'}, \Phi, \rho_{J'})e^{-t}$,
  \begin{align}
    \hat{B}_{J,J'} \gtrsim A_J - A_{J'} - V_{m,J'} - \frac{t}{m}.
  \end{align}

\end{proof}

\section{Proofs of Analyses}
\subsection{Proofs for Bias Analyses}
\paragraph{Properties of Lepski's method}
To prove \zcref{lem:population-lepski-condition}, we leverage the following lemma.
\begin{lemma}\label{lem:population-bias-approx}
  Fix $\phi$ and $(X, Y) \in \mathcal{P}$. Assume that $A_J(\phi; X, Y) \asymp J^{-2\alpha}$. For any constant $c > 1$, there exists $C > 0$ such that for any $J, J' \in \NaturalSet$ where $J' \ge CJ$,
  \begin{align}
    B_{J,J'}(\phi; X, Y) \le cA_J(\phi; X, Y).
  \end{align}
  Furthermore, for $c \ge 4$, this inequality is satisfied with $C = 1$.  Moreover, there exists a constant $c > 0$ such that for any constant $C > 0$ and for any $J, J' \in \NaturalSet$ where $J' \ge CJ$,
  \begin{align}
    B_{J,J'}(\phi; X, Y) \ge \ab(1 - cC^{-\alpha})A_J(\phi; X, Y).
  \end{align}
\end{lemma}
\begin{proof}[Proof of \zcref{lem:population-bias-approx}]
  We use the shorthands $B_{J,J'} = B_{J,J'}(\phi; X, Y)$ and $A_J = A_J(\phi; X, Y)$.
  \paragraph*{Upper bound.} By the triangle inequality, for any $J,J' \in \NaturalSet$,
  \begin{align}
    B_{J,J'} \le A_J\ab(1 + \frac{A_{J'}^{1/2}}{A_J^{1/2}})^2.
  \end{align}
  By the assumption of $A_J \asymp J^{-2\alpha}$, there exists a constant $C' > 0$ such that
  \begin{align}
    \frac{A_{J'}^{1/2}}{A_J^{1/2}} \le C'\ab(\frac{J}{J'})^{\alpha} \le C'C^{-\alpha}.
  \end{align}
  Since $c > 1$, we can choose $C > 0$ such that $1 + C'C^{-\alpha} \le \sqrt{c}$, confirming the upper bound. The further statement follows from the fact that for any $J' \ge J$, $A_{J'} \le A_{J}$.

  \paragraph*{Lower bound.} By the reverse triangle inequality, for any $J,J' \in \NaturalSet$,
  \begin{align}
    B_{J,J'} \ge A_J\ab(1 - \frac{A_{J'}^{1/2}}{A_J^{1/2}})^2.
  \end{align}
  By the same argument as the upper bound, there exists a constant $c > 0$ such that
  \begin{align}
    B_{J,J'} \ge A_J\ab(1 - cC^{-\alpha}/2)^2.
  \end{align}
  For $C>0$ such that $1 - cC^{-\alpha}/2 \in (0,1)$, we have $(1 - cC^{-\alpha}/2)^2 \le 1 - cC^{-\alpha}$, confirming the lower bound.
\end{proof}

Now, we prove \zcref{lem:population-lepski-condition}.
\begin{proof}[Proof of \zcref{lem:population-lepski-condition}]
  Write $A_J$ for $A_J(\phi; X, Y)$ throughout. By assumption, there exist constants $0 < c_\ell \le c_u < \infty$ such that
  \begin{align}
    c_\ell J^{-2\alpha} \le A_J \le c_u J^{-2\alpha} \for \all J \in \NaturalSet.
  \end{align}
  Define $J_0 = (m / \ln m)^{1/(2\alpha+1)}$; a direct computation gives $V_{m, J_0} = J_0 \ln(m)/m = J_0^{-2\alpha}$. Hence, with this $J_0$, we have $V_{m, J_0} \asymp A_{J_0}$ due to assumption.

  By \zcref{lem:population-bias-approx}, for some constant $c \ge 4$, for any $J' \ge J$,
  \begin{align}
    B_{J,J'} \le cA_{J} \le c c_u J^{-2\alpha}.
  \end{align}
  We can choose $J$ such that $c c_u J^{-2\alpha} \le J_0^{-2\alpha} = V_{m, J_0}$, satisfying the Lepski condition. Hence, for such a $J$, we have $J^*_m(\phi; X, Y) \le J \lesssim J_0$.

  By \zcref{lem:population-bias-approx}, for some constant $c \in (0,1)$, there exists $C > 0$ such that for any $J, J' \in \NaturalSet$ where $J' \ge CJ$,
  \begin{align}
    B_{J,J'} \ge cA_{J} \ge c c_\ell J^{-2\alpha}.
  \end{align}
  We can choose $J$ such that $c c_\ell J^{-2\alpha} \ge J_0^{-2\alpha} = V_{m, J_0}$, breaking the Lepski condition. Hence, for such a $J$, we have $J^*_m(\phi; X, Y) \ge J \gtrsim J_0$.
\end{proof}

\paragraph{Proofs for sieved least square estimator}
We use three concentration inequalities for the empirical $L^2$-norm $\Vab*{\cdot}_{L^2_k(X)}^2$, the empirical inner product to the noise $\ab<\cdot, \epsilon>_k$, and the absolute sum of the noise $\frac{1}{k}\sum_{i=1}^k |\epsilon_i|$.
\begin{lemma}\label{lem:squared-error-difference-concentration}
  Let $X$ be a random variable on $\mathcal{X}$. For any fixed measurable function $h \colon \mathcal{X} \to [-1,1]$ and any $t > 0$, with probability at least $1 - e^{-t}$,
  \begin{align}
    \Vab{h}_{L^2_{k}(X)}^2 - \Vab{h}_{L^2(X)}^2 \le \sqrt{\frac{2 t}{k} \Vab{h}_{L^2(X)}^2} + \frac{t}{3k}.
  \end{align}
  Moreover, with probability at least $1 - e^{-t}$,
  \begin{align}
    \Vab{h}_{L^2(X)}^2 - \Vab*{h}_{L^2_{k}(X)}^2 \le \sqrt{\frac{2 t}{k} \Vab*{h}_{L^2(X)}^2} + \frac{t}{3k}.
  \end{align}
\end{lemma}
\begin{lemma}\label{lem:sub-gaussian-noise-concentration}
  Let $X \in \RealSet$ be a random variable. Let $\epsilon$ be sub-gaussian and mean-zero independent random variables such that their variance proxy is at most $\sigma^2 > 0$. Then, conditioned on $X$, with probability at least $1-e^{-t}$,
  \begin{align}
    \ab<X, \epsilon>_k \le \sqrt{\frac{2\sigma^2t}{k}}\Vab{X}_{L^2_k}.
  \end{align}
\end{lemma}
\begin{lemma}\label{lem:sub-gaussian-absolute-sum-concentration}
  Let $\epsilon_1,\dots,\epsilon_k$ be sub-gaussian and mean-zero independent random variables such that their variance proxy is at most $\sigma^2 > 0$. Then, with probability at least $1-e^{-t}$,
  \begin{align}
    \frac{1}{k}\sum_{i=1}^k \ab|\epsilon_i| \le \sigma\ab(\frac{2t}{k} + 2\ln 2)^{1/2}.
  \end{align}
\end{lemma}

From the definition of the Minkowski--Bouligand dimension, for $c_2 > 1$, there is a sufficiently small $c_1 > 0$ such that with $\epsilon_m = \frac{c_1}{m}$, $\ln N(\epsilon_m, \mathcal{F}_J, \Vab*{\cdot}_{L^\infty}) \le c_2J\ln(m)$ for any $m \ge 1$. Let $N_{m,J} = N(\epsilon_m, \mathcal{F}_J, \Vab*{\cdot}_{L^\infty})$ and $f_1,...,f_{N_{m,J}}$ be an $\epsilon_m$-cover of $\mathcal{F}_J$ in $\Vab*{\cdot}_{L^\infty}$. Let $\hat{f}_{J,\epsilon_m}$ be the closest function from these to $\hat{f}_J$ in $\Vab*{\cdot}_{L^\infty}$.

We now prove \zcref{lem:sieved-LS-oracle,lem:sieved-LS-lower}.
\begin{proof}[Proof of \zcref{lem:sieved-LS-oracle}]
  We can decompose the squared error as
  \begin{multline}
    \Vab{\hat{f}_J \circ \phi - \Mean[Y | X]}_{L^2_{m}(X)}^2 = \Vab{f_J \circ \phi - \Mean[Y | X]}_{L^2_{m}(X)}^2 \\ + \frac{2}{m}\sum_{i=1}^{m}\ab((\hat{f}_J\circ\phi)(X_i) - (f_J\circ\phi)(X_i))\epsilon_i \\ + \frac{1}{m}\sum_{i=1}^{m}\ab((\hat{f}_J\circ\phi)(X_i) - Y_i)^2 - \frac{1}{m}\sum_{i=1}^{m}\ab((f_J\circ\phi)(X_i) - Y_i)^2,
  \end{multline}
  where $\epsilon_i = Y_i - \Mean[Y | X]$. Since $\hat{f}_J$ is an empirical minimizer over $\mathcal{F}_J$ and $f_J \in \mathcal{F}_J$, we have
  \begin{multline}
    \Vab{\hat{f}_J \circ \phi - \Mean[Y | X]}_{L^2_{m}(X)}^2 \le \\ \Vab{f_J \circ \phi - \Mean[Y | X]}_{L^2_{m}(X)}^2 + \frac{2}{m}\sum_{i=1}^{m}\ab((\hat{f}_J\circ\phi)(X_i) - (f_J\circ\phi)(X_i))\epsilon_i.
  \end{multline}
  For the second term in the right-hand side, we have
  \begin{multline}
    \frac{1}{m}\sum_{i=1}^{m}\ab((\hat{f}_J\circ\phi)(X_i) - (f_J\circ\phi)(X_i))\epsilon_i \le  \\ \frac{\epsilon_m}{m}\sum_{i=1}^{m}\ab|\epsilon_i| + \frac{1}{m}\sum_{i=1}^{m}\ab((\hat{f}_{J,\epsilon_m}\circ\phi)(X_i) - (f_J\circ\phi)(X_i))\epsilon_i.
  \end{multline}
  By \zcref{lem:sub-gaussian-absolute-sum-concentration}, \zcref{lem:sub-gaussian-noise-concentration}, union bound over the $\epsilon_m$-covers, and triangle inequality, we have with probability at least $1 - 2e^{-t}$,
  \begin{align}
    \frac{\epsilon_m}{m}\sum_{i=1}^{m}\ab|\epsilon_i| \le \sigma\epsilon_m\ab(\frac{2t}{m} + 2\ln 2)^{1/2},
  \end{align}
  and
  \begin{multline}
    \frac{1}{m}\sum_{i=1}^{m}\ab((\hat{f}_{J,\epsilon_m}\circ\phi)(X_i) - (f_J\circ\phi)(X_i))\epsilon_i \le \\ \sqrt{\frac{2\sigma^2(\ln(N_{m,J}) + t)}{m}}\ab(\Vab{\hat{f}_J \circ \phi - f_J \circ \phi}_{L^2_{m}(X)} + \epsilon_m).
  \end{multline}
  Consequently, we have with probability at least $1 - 2e^{-t}$,
  \begin{multline}
    \Vab{\hat{f}_J \circ \phi - \Mean[Y | X]}_{L^2_{m}(X)}^2 \le \\ \Vab{f_J \circ \phi - \Mean[Y | X]}_{L^2_{m}(X)}^2 + 2\sigma\epsilon_m\ab(\frac{2t}{m} + 2\ln 2)^{1/2} \\ + \sqrt{\frac{8\sigma^2(c_2J\ln(m) + t)}{m}}\ab(\Vab{\hat{f}_J \circ \phi - f_J \circ \phi}_{L^2_{m}(X)} + \epsilon_m).
  \end{multline}
  By the triangle inequality, the AM-GM inequality, and the fact that $\sqrt{1 + x} \le 1 + x$ for $x \ge 0$, for any $\eta \in (0,1)$, we have
  \begin{multline}
    (1-\eta)\Vab{\hat{f}_J \circ \phi - \Mean[Y | X]}_{L^2_{m}(X)}^2 \le \\ (1+\eta)\Vab{f_J \circ \phi - \Mean[Y | X]}_{L^2_{m}(X)}^2 + \frac{6\sigma^2 c_2J\ln(m)}{\eta m} \\ + \ab(2\sqrt{2\ln(2)}\sigma c_1 + \frac{\eta c_1^2}{m})\frac{1}{m} + \ab(\frac{2\sqrt{2}\sigma c_1}{\ln^{1/2}(2)}+\frac{6\sigma^2}{\eta})\frac{t}{m}. \label{eq:sieved-LS-oracle-1}
  \end{multline}

  Application of \zcref{lem:squared-error-difference-concentration} to the first term in the right-hand side of \zcref{eq:sieved-LS-oracle-1} yields with probability at least $1 - e^{-t}$,
  \begin{multline}
    \Vab{f_J \circ \phi - \Mean[Y | X]}_{L^2_{m}(X)}^2 \le \\ \Vab{f_J \circ \phi - \Mean[Y | X]}_{L^2(X)}^2 + \sqrt{\frac{2 t}{m} \Vab{f_J \circ \phi - \Mean[Y | X]}_{L^2(X)}^2} + \frac{t}{3m}.
  \end{multline}
  The AM-GM inequality $\sqrt{2tv/m} \le v + t/(2m)$ yields with probability at least $1 - e^{-t}$,
  \begin{align}
    \Vab{f_J \circ \phi - \Mean[Y | X]}_{L^2_{m}(X)}^2 \le 2\Vab{f_J \circ \phi - \Mean[Y | X]}_{L^2(X)}^2 + \frac{5t}{6m}. \label{eq:sieved-LS-oracle-2}
  \end{align}
  Combining \zcref{eq:sieved-LS-oracle-1} and \zcref{eq:sieved-LS-oracle-2} yields the claim.
\end{proof}
\begin{proof}[Proof of \zcref{lem:sieved-LS-lower}]
  By the triangle and reverse triangle inequalities, we have
  \begin{align}
    \Vab{\hat{f}_J \circ \phi - \Mean[Y | X]}_{L^2_{m}(X)}^2 \ge& \ab(\Vab{\hat{f}_{J,\epsilon_m} \circ \phi - \Mean[Y | X]}_{L^2_{m}(X)} - \Vab{\hat{f}_J\circ\phi - \hat{f}_{J,\epsilon_m}\circ\phi}_{L^2_m(X)})^2 \\
    \ge& \ab(\Vab{\hat{f}_{J,\epsilon_m} \circ \phi - \Mean[Y | X]}_{L^2_{m}(X)} - \epsilon_m)^2 \\
    \ge& \frac{1}{2}\Vab{\hat{f}_{J,\epsilon_m} \circ \phi - \Mean[Y | X]}_{L^2_{m}(X)}^2 - \epsilon_m^2,
  \end{align}
  Application of the union bound to \zcref{lem:squared-error-difference-concentration} gives with probability at least $1-e^{-t}$,
  \begin{multline}
    \Vab{\hat{f}_{J,\epsilon_m} \circ \phi - \Mean[Y | X]}_{L^2_{m}(X)}^2 \ge \Vab{\hat{f}_{J,\epsilon_m} \circ \phi - \Mean[Y | X]}_{L^2(X)}^2 \\ - \sqrt{\frac{2\ln N_{m,J} + 2t}{m}}\Vab{\hat{f}_{J,\epsilon_m} \circ \phi - \Mean[Y | X]}_{L^2(X)} - \frac{4\ln N_{m,J}}{3m} - \frac{t}{3m}.
  \end{multline}
  The AM-GM inequality yields with probability at least $1 - e^{-t}$,
  \begin{align}
    \Vab{\hat{f}_{J,\epsilon_m} \circ \phi - \Mean[Y | X]}_{L^2_{m}(X)}^2 \ge \frac{1}{2}\Vab{\hat{f}_{J,\epsilon_m} \circ \phi - \Mean[Y | X]}_{L^2(X)}^2 - \frac{4\ln N_{m,J}}{3m} - \frac{4t}{3m}.
  \end{align}
  By the definition of $\hat{f}_{J,\epsilon_m}$, we have
  \begin{align}
    \Vab{\hat{f}_{J,\epsilon_m} \circ \phi - \Mean[Y | X]}_{L^2_{m}(X)}^2 \ge \frac{1}{2}\ab(\Vab{\hat{f}_{J} \circ \phi - \Mean[Y | X]}_{L^2(X)} - \epsilon_m)^2 - \frac{4\ln N_{m,J}}{3m} - \frac{4t}{3m}.
  \end{align}
  Noting that $\hat{f}_{J} \in \mathcal{F}_J$ almost surely, we have
  \begin{align}
    \Vab{\hat{f}_{J} \circ \phi - \Mean[Y | X]}_{L^2(X)}^2 \ge A_J(\phi; X, Y) \mbox{ almost surely}.
  \end{align}
  By the triangle inequality, we have
  \begin{align}
    \ab(A_J(\phi; X, Y) - \epsilon_m)^2 \ge \frac{1}{2}A_J(\phi; X, Y) - \epsilon_m^2.
  \end{align}
  Consequently, we have with probability at least $1-e^{-t}$,
  \begin{align}
    \Vab{\hat{f}_{J} \circ \phi - \Mean[Y | X]}_{L^2_{m}(X)}^2 \ge \frac{1}{4}A_J(\phi; X, Y) - \frac{3\epsilon_m^2}{2} - \frac{2\ln N_{m,J}}{3m} - \frac{2t}{3m}.
  \end{align}
  Using the upper bound on $\ln N_{m,J}$ and the definition of $\epsilon_m$, we get the desired result.
\end{proof}

\subsection{Proofs for Bias Analyses with Learned Feature Extractor}

\paragraph{Proofs for sieved least-square estimator}
Consider $\epsilon$-cover of $\Phi$ in \zcref{asm:complex-phi}, denoted as $\phi_1,...,\phi_{N_{\epsilon}}$, where $N_{\epsilon} = N(\epsilon, \Phi, \rho_J)$. Let $\hat{i}_{\epsilon} = \argmin_{i \in [N_{\epsilon}]}\rho_J(\hat\phi, \phi_i)$ and $\hat\phi_{\epsilon} = \phi_{\hat{i}_{\epsilon}}$.

From the definition of the Minkowski--Bouligand dimension, for $c_2 > 1$, there is a sufficiently small $c_1 > 0$ such that with $\epsilon_m = \frac{c_1}{m}$, $\ln N(\epsilon_m, \mathcal{F}_J, \Vab*{\cdot}_{L^\infty}) \le c_2J\ln(m)$ for any $m \ge 1$. Let $N_{m,J} = N(\epsilon, \mathcal{F}_J, \Vab*{\cdot}_{L^\infty})$ and $f_1,...,f_{N_{m,J}}$ be a $\epsilon_m$-cover of $\mathcal{F}_J$ in $\Vab*{\cdot}_{L^\infty}$. Let $\hat{f}_{J,\epsilon_m}$ be the closest function from these functions to $\hat{f}_J$ in $\Vab*{\cdot}_{L^\infty}$.

\begin{proof}[Proof of \zcref{lem:sieved-LS-oracle-learned}]
  We can decompose the squared error as
  \begin{multline}
    \Vab{\hat{f}_J \circ \hat\phi - \Mean[Y | X]}_{L^2_{m}(X)}^2 = \Vab{f_J \circ \hat\phi - \Mean[Y | X]}_{L^2_{m}(X)}^2 \\ + \frac{2}{m}\sum_{i=1}^{m}\ab((\hat{f}_J\circ\hat\phi)(X_i) - (f_J\circ\hat\phi)(X_i))\epsilon_i \\ + \frac{1}{m}\sum_{i=1}^{m}\ab((\hat{f}_J\circ\hat\phi)(X_i) - Y_i)^2 - \frac{1}{m}\sum_{i=1}^{m}\ab((f_J\circ\hat\phi)(X_i) - Y_i)^2,
  \end{multline}
  where $\epsilon_i = Y_i - \Mean[Y | X]$. Since $\hat{f}_J$ is an empirical minimizer over $\mathcal{F}_J$ and $f_J \in \mathcal{F}_J$, we have
  \begin{multline}
    \Vab{\hat{f}_J \circ \hat\phi - \Mean[Y | X]}_{L^2_{m}(X)}^2 \le \\ \Vab{f_J \circ \hat\phi - \Mean[Y | X]}_{L^2_{m}(X)}^2 + \frac{2}{m}\sum_{i=1}^{m}\ab((\hat{f}_J\circ\hat\phi)(X_i) - (f_J\circ\hat\phi)(X_i))\epsilon_i.
  \end{multline}
  For the second term in the right-hand side, we have
  \begin{multline}
    \frac{1}{m}\sum_{i=1}^{m}\ab((\hat{f}_J\circ\hat\phi)(X_i) - (f_J\circ\hat\phi)(X_i))\epsilon_i \le  \\ \frac{\epsilon_m + 2\epsilon}{m}\sum_{i=1}^{m}\ab|\epsilon_i| + \frac{1}{m}\sum_{i=1}^{m}\ab((\hat{f}_{J,\epsilon_m}\circ\hat\phi_{\epsilon})(X_i) - (f_J\circ\hat\phi_{\epsilon})(X_i))\epsilon_i.
  \end{multline}
  By \zcref{lem:sub-gaussian-absolute-sum-concentration}, \zcref{lem:sub-gaussian-noise-concentration}, union bound over the $\epsilon$-covers, and triangle inequality, we have with probability at least $1 - (1+N_{\epsilon})e^{-t}$,
  \begin{align}
    \frac{\epsilon}{m}\sum_{i=1}^{m}\ab|\epsilon_i| \le \sigma\ab(\epsilon_m + 2\epsilon)\ab(\frac{2t}{m} + 2\ln 2)^{1/2},
  \end{align}
  and
  \begin{multline}
    \frac{1}{m}\sum_{i=1}^{m}\ab((\hat{f}_{J,\epsilon}\circ\hat\phi_\epsilon)(X_i) - (f_J\circ\hat\phi_\epsilon)(X_i))\epsilon_i \le \\ \sqrt{\frac{2\sigma^2(\ln(N_{m,J}) + t)}{m}}\ab(\Vab{\hat{f}_J \circ \hat\phi - f_J \circ \hat\phi_{\epsilon}}_{L^2_{m}(X)} + \epsilon_m).
  \end{multline}
  Consequently, we have with probability at least $1 - (1+N_{\epsilon})e^{-t}$,
  \begin{multline}
    \Vab{\hat{f}_J \circ \hat\phi - \Mean[Y | X]}_{L^2_{m}(X)}^2 \le \\ \Vab{f_J \circ \hat\phi - \Mean[Y | X]}_{L^2_{m}(X)}^2 + 2\sigma(\epsilon_m + 2\epsilon)\ab(\frac{2t}{m} + 2\ln 2)^{1/2} \\ + \sqrt{\frac{8\sigma^2(c_2J\ln(m) + t)}{m}}\ab(\Vab{\hat{f}_J \circ \hat\phi - f_J \circ \hat\phi_\epsilon}_{L^2_{m}(X)} + \epsilon_m + \epsilon).
  \end{multline}
  By the triangle inequality, the AM-GM inequality, and the fact that $\sqrt{1 + x} \le 1 + x$ for $x \ge 0$, for any $\eta \in (0,1)$, we have
  \begin{multline}
    (1-\eta)\Vab{\hat{f}_J \circ \hat\phi - \Mean[Y | X]}_{L^2_{m}(X)}^2 \le \\ (2+3\eta)\Vab{f_J \circ \hat\phi_\epsilon - \Mean[Y | X]}_{L^2_{m}(X)}^2 + \frac{6\sigma^2 c_2J\ln(m)}{\eta m} \\ + \ab(2\sqrt{2\ln(2)}\sigma c_1 + \frac{3\eta c_1^2}{m})\frac{1}{m} + (2+3\eta)\epsilon + \ab(\frac{2\sqrt{2}\sigma c_1}{\ln^{1/2}(2)}+\frac{6\sigma^2}{\eta})\frac{t}{m}. \label{eq:sieved-LS-oracle-1-learned}
  \end{multline}

  Application of \zcref{lem:squared-error-difference-concentration} with union bound over $\epsilon$-cover of $\Phi$ to the first term in the right-hand side of \zcref{eq:sieved-LS-oracle-1-learned} yields with probability at least $1 - N_\epsilon e^{-t}$,
  \begin{multline}
    \Vab{f_J \circ \hat\phi_\epsilon - \Mean[Y | X]}_{L^2_{m}(X)}^2 \le \\ \Vab{f_J \circ \hat\phi_\epsilon - \Mean[Y | X]}_{L^2(X)}^2 + \sqrt{\frac{2 t}{m} \Vab{f_J \circ \hat\phi_\epsilon - \Mean[Y | X]}_{L^2(X)}^2} + \frac{t}{3m}.
  \end{multline}
  The AM-GM inequality $\sqrt{2tv/m} \le v + t/(2m)$ yields with probability at least $1 - N_{\epsilon}e^{-t}$,
  \begin{align}
    \Vab{f_J \circ \hat\phi_\epsilon - \Mean[Y | X]}_{L^2_{m}(X)}^2 \le 4\Vab{f_J \circ \hat\phi - \Mean[Y | X]}_{L^2(X)}^2 + 2\epsilon^2 + \frac{5t}{6m}. \label{eq:sieved-LS-oracle-2-learned}
  \end{align}
  Combining \zcref{eq:sieved-LS-oracle-1-learned} and \zcref{eq:sieved-LS-oracle-2-learned} yields the claim.
\end{proof}
\begin{proof}[Proof of \zcref{lem:sieved-LS-lower-learned}]
  By the reverse triangle and triangle inequalities, we have
  \begin{align}
    \Vab{\hat{f}_J\circ\hat\phi - \Mean[Y|X]}_{L^2_m(X)}^2 \ge& \ab(\Vab{\hat{f}_J\circ\hat\phi_{\epsilon} - \Mean[Y|X]}_{L^2_m(X)} - \epsilon)^2 \\
    \ge& \ab(\Vab{\hat{f}_{J,\epsilon_m}\circ\hat\phi_{\epsilon} - \Mean[Y|X]}_{L^2_m(X)} - \epsilon_m - \epsilon)^2 \\
    \ge& \frac{1}{2}\Vab{\hat{f}_{J,\epsilon_m}\circ\hat\phi_{\epsilon} - \Mean[Y|X]}_{L^2_m(X)}^2- \epsilon_m^2 - \epsilon^2.
  \end{align}

  Application of the union bound to \zcref{lem:squared-error-difference-concentration} gives with probability at least $1-N_{\epsilon}e^{-t}$,
  \begin{multline}
    \Vab{\hat{f}_{J,\epsilon_m} \circ \hat\phi_{\epsilon} - \Mean[Y | X]}_{L^2_{m}(X)}^2 \ge \Vab{\hat{f}_{J,\epsilon_m} \circ \hat\phi_{\epsilon}  - \Mean[Y | X]}_{L^2(X)}^2 \\ - \sqrt{\frac{2\ln N_{m,J} + 2t}{m}}\Vab{\hat{f}_{J,\epsilon_m} \circ \hat\phi_{\epsilon} - \Mean[Y | X]}_{L^2(X)} - \frac{\ln N_{m,J}}{3m} - \frac{t}{3m}.
  \end{multline}
  The AM-GM inequality yields with probability at least $1 - N_{\epsilon}e^{-t}$,
  \begin{align}
    \Vab{\hat{f}_{J,\epsilon_m} \circ \hat\phi_{\epsilon} - \Mean[Y | X]}_{L^2_{m}(X)}^2 \ge \frac{1}{2}\Vab{\hat{f}_{J,\epsilon_m} \circ \hat\phi_{\epsilon} - \Mean[Y | X]}_{L^2(X)}^2 - \frac{4\ln N_{m,J}}{3m} - \frac{4t}{3m}.
  \end{align}
  By the definitions of $\hat{f}_{J,\epsilon_m}$ and $\hat\phi_{\epsilon}$, we have
  \begin{multline}
    \Vab{\hat{f}_{J,\epsilon_m} \circ \hat\phi_{\epsilon} - \Mean[Y | X]}_{L^2_{m}(X)}^2 \ge \\ \frac{1}{2}\ab(\Vab{\hat{f}_{J} \circ \hat\phi - \Mean[Y | X]}_{L^2(X)} - \epsilon_m - \epsilon)^2 - \frac{4\ln N_{m,J}}{3m} - \frac{4t}{3m}.
  \end{multline}
  Noting that $\hat{f}_{J} \in \mathcal{F}_J$ almost surely, we have
  \begin{align}
    \Vab{\hat{f}_{J} \circ \hat\phi - \Mean[Y | X]}_{L^2(X)}^2 \ge A_J(\hat\phi; X, Y) \mbox{ almost surely}.
  \end{align}
  By the triangle inequality, we have
  \begin{align}
    \ab(A_J(\hat\phi; X, Y) - \epsilon_m - \epsilon)^2 \ge \frac{1}{2}A_J(\hat\phi; X, Y) - \epsilon^2_m - \epsilon^2.
  \end{align}
  Consequently, we have with probability at least $1-N_{\epsilon}e^{-t}$,
  \begin{align}
    \Vab{\hat{f}_{J} \circ \hat\phi - \Mean[Y | X]}_{L^2_{m}(X)}^2 \ge \frac{1}{4}A_J(\hat\phi; X, Y) - \frac{3(\epsilon^2_m + \epsilon^2)}{2} - \frac{2\ln N_{m,J}}{3m} - \frac{2t}{3m}.
  \end{align}
  Using the upper bound on $\ln N_{m,J}$ and the definition of $\epsilon_m$, we get the desired result.
\end{proof}

\subsection{Proofs for Concentration Inequalities}
\begin{proof}[Proof of \zcref{lem:squared-error-difference-concentration}]
  Let $X_1, \dots, X_k$ be i.i.d.\ copies of $X$ and set $g = h^2 \colon \mathcal{X} \to [0, 1]$. Then
  \begin{align}
    \Vab*{h}_{L^2_k(X)}^2 - \Vab*{h}_{L^2(X)}^2 = \frac{1}{k}\sum_{i=1}^{k}\ab(g(X_i) - \Mean[g(X)]).
  \end{align}
  The summands $g(X_i) - \Mean[g(X)]$ are i.i.d., bounded in $[-1, 1]$, and satisfy
  \begin{align}
    \Var[g(X)] \le \Mean[g(X)^2] = \Mean[h(X)^4] \le \Mean[h(X)^2] = \Vab*{h}_{L^2(X)}^2,
  \end{align}
  where $h^4 \le h^2$ holds pointwise since $|h| \le 1$. Applying Bernstein's inequality to the i.i.d.\ summands $g(X_i) - \Mean[g(X)]$ yields, with probability at least $1 - e^{-t}$,
  \begin{align}
    \frac{1}{k}\sum_{i=1}^{k}\ab(g(X_i) - \Mean[g(X)])
    \le \sqrt{\frac{2\Var[g(X)]\,t}{k}} + \frac{t}{3k}
    \le \sqrt{\frac{2 t}{k}\Vab*{h}_{L^2(X)}^2} + \frac{t}{3k}.
  \end{align}
  The same derivation is valid even if we exchange $\Vab*{h}_{L^2_{k}(X)}^2$ and $\Vab{h}_{L^2(X)}^2$.
\end{proof}
\begin{proof}[Proof of \zcref{lem:sub-gaussian-noise-concentration}]

  \begin{align}
    & \Mean\ab[\exp\ab(\frac{\lambda}{k}\sum_{i=1}^k h_i\epsilon_i)] \\
    =& \prod_{i=1}^k \Mean\ab[\exp\ab(\frac{\lambda}{k} h_i\epsilon_i)] \\
    =& \prod_{i=1}^k \exp\ab(\frac{\lambda^2\sigma^2 h_i^2}{2k^2}) \\
    =& \exp\ab(\frac{\lambda^2\sigma^2}{2k} \cdot \frac{1}{k}\sum_{i=1}^k h_i^2).
  \end{align}

  By the Chernoff bound, we have
  \begin{align}
    \p\Bab{\frac{1}{k}\sum_{i=1}^k h_i\epsilon_i > t} \le \inf_{\lambda > 0} \exp\ab(\frac{\lambda^2\sigma^2}{2k} \cdot \frac{1}{k}\sum_{i=1}^k h_i^2 - \lambda t) \le \exp\ab(-\frac{t^2k}{2\sigma^2}\ab(\frac{1}{k}\sum_{i=1}^k h_i^2)^{-1}).
  \end{align}
  Choosing $t$ appropriately yields the claim.
\end{proof}
\begin{proof}[Proof of \zcref{lem:sub-gaussian-absolute-sum-concentration}]
  For any $\lambda > 0$ and $x \in \RealSet$, the inequality $e^{\lambda|x|} \le e^{\lambda x} + e^{-\lambda x}$ holds, so the sub-gaussian assumption gives
  \begin{align}
    \Mean\ab[e^{\lambda|\epsilon_i|}] \le \Mean\ab[e^{\lambda\epsilon_i}] + \Mean\ab[e^{-\lambda\epsilon_i}] \le 2e^{\lambda^2\sigma^2/2}.
  \end{align}
  By independence,
  \begin{align}
    \Mean\ab[\exp\ab(\frac{\lambda}{k}\sum_{i=1}^k |\epsilon_i|)] \le 2^k \exp\ab(\frac{\lambda^2\sigma^2}{2k}).
  \end{align}
  The Chernoff bound then gives, for any $s > 0$,
  \begin{align}
    \p\ab(\frac{1}{k}\sum_{i=1}^k |\epsilon_i| > s)
    \le \inf_{\lambda > 0} \exp\ab(-\lambda ks + k\ln 2 + \frac{k\lambda^2\sigma^2}{2}).
  \end{align}
  The infimum over $\lambda > 0$ is attained at $\lambda^* = s/\sigma^2$, yielding
  \begin{align}
    \p\ab(\frac{1}{k}\sum_{i=1}^k |\epsilon_i| > s) \le \exp\ab(-\frac{ks^2}{2\sigma^2} + k\ln 2).
  \end{align}
  Setting $s = \sigma\ab(\frac{2t}{k} + 2\ln 2)^{1/2}$ gives $\frac{ks^2}{2\sigma^2} = t + k\ln 2$, so the right-hand side equals $e^{-t}$.
\end{proof}

\section{Experiments}\label{app:experiments}

\subsection{Tasks and Architectures}
We meta-train on three settings: Omniglot~\autocite{omniglot2015}
$5$-way and $20$-way $1$-shot on a Conv4-64 backbone over
$28\!\times\!28$ grayscale inputs, and Mini-ImageNet~\autocite{vinyals2016matching,ravi2017optimization} $5$-way
$1$-shot on a Conv4-128 backbone over $84\!\times\!84$ RGB inputs,
each with five queries per task. Omniglot features transfer to
SVHN~\autocite{svhn2011}, while Mini-ImageNet features transfer to
CIFAR-10 and CIFAR-100~\autocite{cifar2009}.
We set the pre-training sample size $m\in\{1{,}000, 2{,}000, 4{,}000, 8{,}000, 16{,}000\}$ for Omniglot and $m\in\{8{,}000, 16{,}000, 24{,}000, 32{,}000, 40{,}000\}$ for Mini-ImageNet, and the fine-tuning sample size $n\in\{100, 500, 1{,}000, 2{,}500, 5{,}000\}$.

\subsection{Algorithms, Regularizers, and Optimization}
We compare four meta-learners spanning the gradient-based and metric-based families: first- and second-order MAML~\autocite{finnModelAgnosticMetaLearningFast2017}, Prototypical Networks~\autocite{snell2017prototypical}, and R2-D2~\autocite{bertinetto2018metalearning}.
Meta-training adds a regularization term $\lambda \cdot \mathcal{R}(\theta)$ to the meta-loss, with $\mathcal{R}$ the spectral norm $\sum_{l} \sigma_{\max}(W_l)$ or the entry-wise $\ell_1$ norm $\sum_{l} \lVert W_l \rVert_{1}$, each globally normalized so that $\lambda$ is comparable across architectures and
penalties.
The outer loop runs $3{,}000$ steps of Adam~\autocite{adam2015} ($\beta_1{=}0.9$, $\beta_2{=}0.95$) at $\text{lr}=3\times10^{-4}$ with meta batch size $8$.
Fine-tuning trains the full backbone plus a fresh linear head for $30$ epochs of Adam at $3\times10^{-4}$.
The regularization coefficient $\lambda$ is selected by grid search over $\{10^{-2}, 10^{-1}, 1, 10, 10^{2}, 10^{3}\}$.
The mean values over three random seeds with standard deviation are reported.

\subsection{Implementation and Compute}
The pipeline is built on \texttt{JAX}~\autocite{jax2018} v0.10 with CUDA 12.9, \texttt{Equinox}~\autocite{equinox2021}, and \texttt{Optax}~\autocite{optax2020}.
Every trial occupies an NVIDIA~H100 ($80$\,GB) in single precision.
The full sweep consumed about $1{,}000$ H100-hours.

\subsection{Additional Results}

\zcref{app:fig:spectral_combined} presents the downstream test error rates on SVHN, CIFAR-10, and CIFAR-100 for CNN models pre-trained on Omniglot (5-way 1-shot and 20-way 1-shot) and Mini-ImageNet.
It can be seen that regularizing model complexity improves downstream sample efficiency in various settings.

In addition to the results with the spectral norm regularization, \zcref{app:fig:l1_combined} displays the results with the $\ell_1$-norm regularization. 
We can observe the improvements by $\ell_1$-norm regularization, although they are not as vivid as those by spectral-norm regularization.

\begin{figure}
  \centering
  \includegraphics[width=\linewidth]{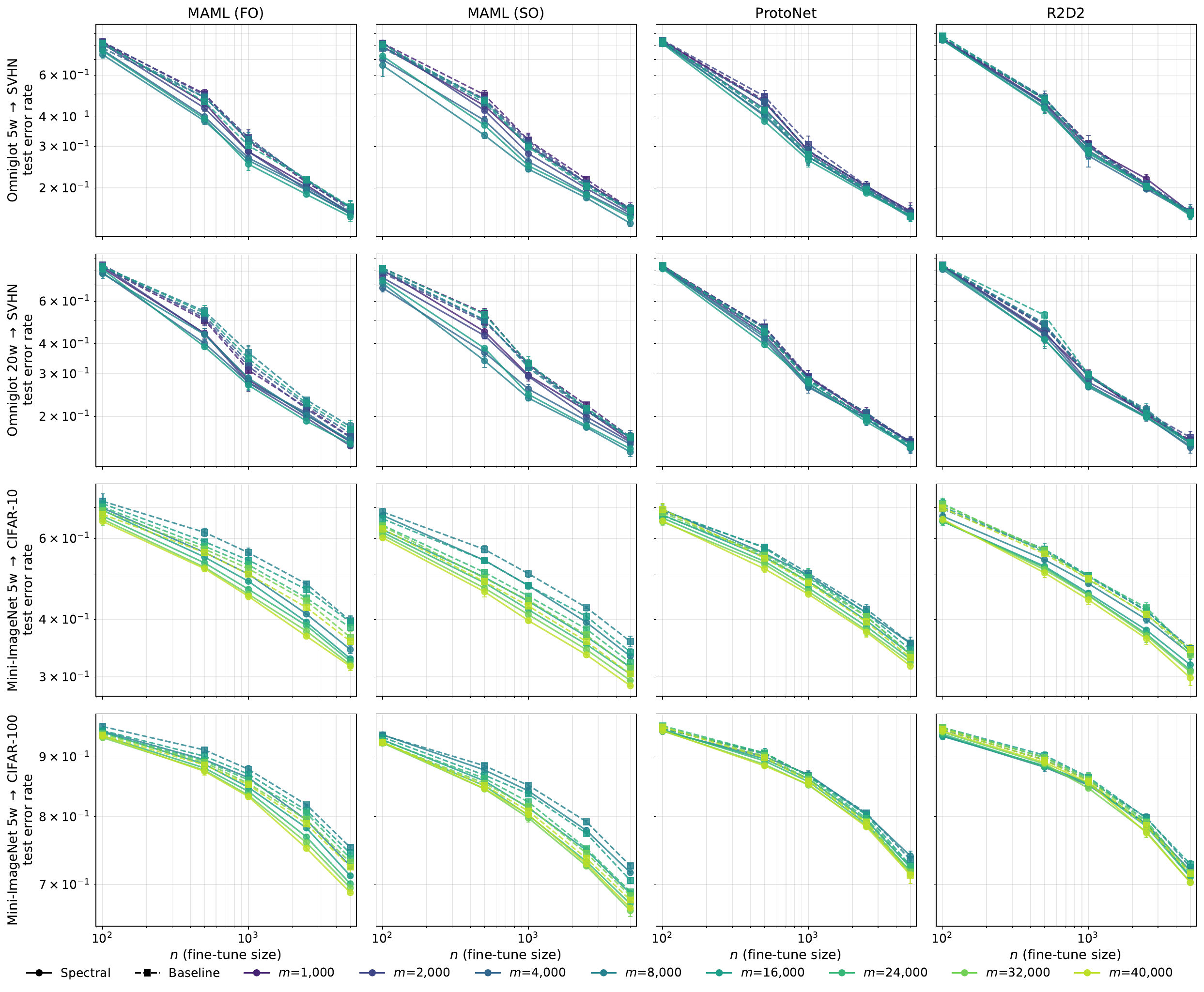}
  \caption{
    Downstream test error rate on SVHN, CIFAR-10, CIFAR-100 (log scale) vs.\ fine-tuning sample size ($n$, log scale) with different pre-training sample size ($m$) for four meta-learning algorithms with and without parameter spectral-norm regularization trained on Omniglot or Mini-ImageNet.
  }
  \label{app:fig:spectral_combined}
\end{figure}

\begin{figure}
  \centering
  \includegraphics[width=\linewidth]{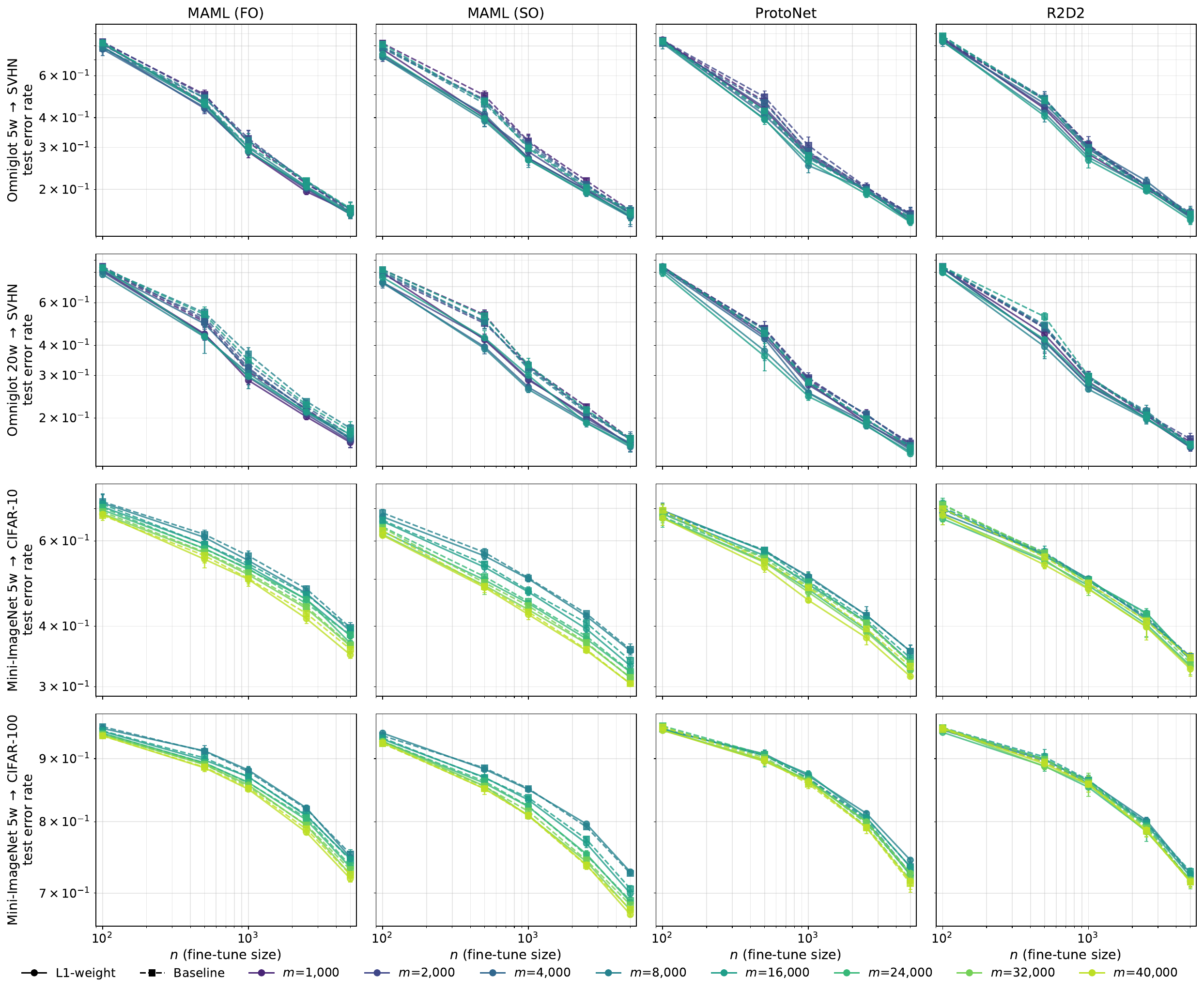}
  \caption{
    Downstream test error rate on SVHN, CIFAR-10, CIFAR-100 (log scale) vs.\ fine-tuning sample size ($n$, log scale) with different pre-training sample size ($m$) for four meta-learning algorithms with and without parameter $\ell_1$-norm regularization trained on Omniglot or Mini-ImageNet.
  }
  \label{app:fig:l1_combined}
\end{figure}


\section{Comprehensive survey of related work}\label{sec:related_work_appendix}

\subsection{Meta-Learning Methodologies}
Meta-learning is a framework that seeks to acquire a learning procedure capable of adapting to future unseen tasks using limited samples.
Rather than focusing on generalization within a single task, it leverages experiences from a collection of past tasks drawn from a task distribution~\autocite{hospedales2021meta,wang2020generalizing}.
A widely adopted taxonomy for these methods consists of a tripartite classification: (i) metric-based methods, which rely on distances and similarities; (ii) optimization-based methods, which incorporate gradient updates in an inner loop; and (iii) model-based methods, which construct the learner itself using mechanisms such as memory or hypernetworks.

\textbf{Metric-based approaches.}
The metric-based family employs a framework for classifying queries based on proximity, attention, or comparison within an embedding space.
As a representative example, Matching Networks proposed one-shot classification via attention between a support set and a query, formalizing the episodic training regime~\autocite{vinyals2016matching}.
Prototypical Networks introduced a concise classifier based on distances to class-specific ``prototypes'' (mean embeddings), providing a clear perspective on meta-representation learning~\autocite{snell2017prototypical}.
Relation Networks enable more expressive comparisons by learning the distance function itself using a neural network~\autocite{sung2018learning}.

\textbf{Optimization-based approaches.}
The optimization-based family views meta-learning as ``learning an initial parameter set or update rule such that a few steps of optimization on an unseen task lead to high performance.''
MAML established a model-agnostic framework by adapting to task-specific parameters via $K$-step gradient descent in the inner loop and optimizing the post-adaptation performance in the outer loop~\autocite{finnModelAgnosticMetaLearningFast2017}.
First-order methods such as FO-MAML and Reptile are categorized as algorithms that avoid second-order derivative computations while shifting initial values in a direction that makes ``simultaneous learning from the same starting point'' easier across tasks~\autocite{nichol2018first}.
Meta-SGD further parameterizes not only the initial values but also the update directions and learning rates, thereby learning a higher-capacity ``way to learn''~\autocite{li2017meta}.
Another line of work replaces the iterative inner-loop adaptation with differentiable
closed-form or rapidly convergent solvers.
R2-D2 introduced a differentiable ridge-regression base learner that constructs
task-specific classifiers on top of learned embeddings, allowing the meta-objective
to be optimized by backpropagating through the solver itself~\autocite{bertinetto2018metalearning}.
This approach occupies an intermediate position between metric-based methods,
which often rely on fixed nearest-neighbor or prototype rules after representation
learning, and gradient-based methods such as MAML, which perform explicit
iterative parameter adaptation.
Additionally, iMAML, which computes meta-gradients using implicit gradients without explicitly unrolling the inner loop, is a representative example of scaling these methods by treating them as bilevel optimization problems~\autocite{rajeswaran2019meta}.

\textbf{Model-based approaches.}
The model-based family implements intra-task adaptation as part of the network's computation using external memory, hypernetworks, or architectures designed to learn the optimizer.
Memory-Augmented Neural Networks (MANN) demonstrated rapid one/few-shot adaptation by using external memory to quickly write and read new information, providing a foundation for meta-learning via model design~\autocite{santoro2016meta}.
Furthermore, the classical lineage of ``learning to learn'' such as ``Optimization as a Model'' (which uses LSTMs to learn update rules), can be understood as a bridge between model-based and optimization-based approaches~\autocite{ravi2017optimization,hochreiter2001learning,andrychowicz2016learning}. SNAIL demonstrated high performance across multiple domains (supervised and reinforcement learning) as a general-purpose meta-learner combining temporal convolutions with attention~\autocite{mishra2017simple}.

\subsection{Meta-Representation Learning: Sharing Representations Across Tasks}

Meta-learning is closely related to transfer learning in that it transfers experiences from numerous tasks to an unseen task~\autocite{pan2009survey}.
Representation learning naturally motivates the acquisition of shared representations in transfer and multi-task learning, based on the general principle that mapping inputs to useful latent representations facilitates learning~\autocite{bengio2013representation}.
Meta-representation learning is characterized by specializing these shared representations for few-shot adaptation within a task, aiming for both statistical and computational efficiency simultaneously.

Classically, Baxter’s model of inductive bias learning clarified the concept of meta-generalization, namely learning a good hypothesis space by observing multiple tasks from a task environment, and served as the starting point for subsequent formalizations~\autocite{baxter2000model}.
In more recent learning theory, Multi-Task Representation Learning (MTRL) has emerged, showing the benefits of learning low-dimensional representations such as dictionaries or feature maps from multiple tasks through generalization error bounds; this provides theoretical support for the acquisition of shared representations in meta-representation learning~\autocite{maurer2016benefit,maurer2013sparse}.

As a theory dealing more explicitly with meta-representation learning, research has provided algorithms and lower bounds for achieving sample-efficient representation estimation and transfer to unseen tasks in settings where a group of linear regression tasks shares a common low-dimensional linear representation~\autocite{tripuraneni2021provable}.
Additionally, some studies analyze the effects of overparameterization on the sample efficiency of meta-representation learning using linear regression sequences, beginning to explain the phenomena observed in deep meta-learning where few-shot adaptation is possible even with large-scale models~\autocite{sun2021towards}.

\subsection{Learning Theory of Meta-Learning}
Learning theory for meta-learning must account for a dual-sampling structure: the extraction of tasks from a task distribution and the sampling of data points within each individual task~\autocite{baxter2000model,hospedales2021meta}.
Recent theoretical studies commonly employ a framework that decomposes excess risk into statistical estimation error, optimization error, and model approximation error, centered around the meta-generalization gap, the discrepancy between the expected risk on unseen tasks and the empirical meta-objective~\autocite{rezazadeh2022unified,wang2024stability}.
In particular, significant progress has been made in precisely analyzing the effects of representation sharing across tasks and how the number of adaptation steps (the inner loop) influences the overall stability of the algorithm~\autocite{huang2022provable,chen2022understanding}.

\textbf{Algorithmic Stability.} This measures the sensitivity of the output to the replacement of a single data point in the training set. By introducing the concept of ``meta-stability'' which accounts for the stability of both the inner and outer loops, this framework provides realistic bounds even for gradient-based methods involving non-convex optimization~\autocite{wang2024stability,bousquet2002stability}.

\textbf{PAC-Bayes Theory.} By introducing a hierarchy of meta-priors and task-specific posteriors, this approach derives bounds dependent on both the task count and sample size~\autocite{pentina2014pac,amit2018meta,rezazadeh2022unified}.

\textbf{Information-Theoretic Approach.} This approach evaluates generalization error using the mutual information between the algorithm's output and the input data, thereby quantifying the dependency on the underlying data distribution~\autocite{chen2021generalization}.

\textbf{Uniform Convergence.} Although this framework utilizes traditional complexity measures, the resulting bounds tend to be loose in the context of deep learning and meta-learning.
Consequently, data-dependent analyses have become the mainstream approach in recent years \autocite{nagarajan2019uniform}.

\subsection{Deep Learning Theory}
Theoretical understanding of deep learning has been developed from several complementary perspectives, including approximation theory, statistical generalization, optimization, representation learning, and scaling laws. 
A classical line of work studies the expressive power and statistical estimation properties of neural networks. Deep ReLU networks are known to approximate rich function classes with rates depending on smoothness, sparsity, compositionality, or intrinsic dimension~\autocite{yarotsky2017error,schmidt-hieberNonparametricRegressionUsing2020}. 
Particularly relevant to our work is the theory of adaptive approximation and estimation by deep networks. 
Suzuki~\autocite{suzuki2019adaptivity} showed that deep ReLU networks achieve minimax optimal rates over Besov and mixed-smooth Besov spaces and can adapt to spatially inhomogeneous smoothness. 
Hayakawa and Suzuki~\autocite{hayakawaMinimaxOptimalitySuperiority2020} further established minimax optimality and the superiority of deep neural network learning over sparse parameter spaces. 
These results clarify an important statistical mechanism behind deep learning: deep nonlinear architectures can exploit hidden structural regularities that are difficult for non-adaptive linear or kernel methods to capture. 
This adaptivity perspective has also been extended to modern architectures, including convolutional and ResNet-type networks, Transformers, and diffusion models~\autocite{oono2019approximation,takakura2023approximation,oko2023diffusion}. 

Another major line of work studies generalization in overparameterized neural networks. 
Since classical capacity bounds based on the raw number of parameters are too pessimistic for modern deep learning, refined analyses have been developed using norms, margins, PAC-Bayes bounds, compression, and algorithm-dependent complexity measures. 
For example, Bartlett et al.~\autocite{bartlett2017spectrally} derived spectrally-normalized margin bounds, and Neyshabur et al.~\autocite{neyshabur2018pac} developed PAC-Bayesian spectrally-normalized bounds. 
At the same time, empirical and theoretical studies of interpolation, benign overfitting, and double descent have shown that the classical bias--variance trade-off does not fully explain modern neural network generalization \autocite{zhang2017understanding,belkin2019reconciling,nakkiran2021deep}. 
These studies mainly concern single-task learning, whereas our work studies how representations learned from multiple source tasks affect the sample complexity of future tasks. 

Optimization theory provides another perspective. 
The neural tangent kernel (NTK) theory shows that infinitely wide neural networks trained by gradient descent can behave like kernel methods~\autocite{jacot2018neural}, and related overparameterization analyses establish global convergence of gradient-based methods under suitable assumptions~\autocite{du2019gradient,allen2019convergence}. 
However, NTK analyses typically describe a lazy-training regime in which features remain nearly fixed during training~\autocite{chizat2019lazy}. 
This perspective alone is insufficient to explain representation learning, where the features themselves are learned. 
Mean-field analyses provide an alternative view in which the distribution of neurons evolves during training and feature learning can occur~\autocite{mei2018mean,sirignano2020mean,woodworth2020kernel}.

Recent work has therefore focused on feature learning beyond fixed-kernel or lazy-training regimes. 
Ba et al.~\autocite{ba2022high} showed that even a single gradient step on the first-layer weights of a two-layer network can improve the learned representation over random features and outperform broad classes of fixed-kernel methods. 
Suzuki et al.~\autocite{suzuki2023feature} analyzed feature learning via mean-field Langevin dynamics and showed that mean-field neural networks can achieve sample-complexity improvements over kernel methods for structured problems such as sparse parity learning. 
More recently, Nishikawa et al.~\autocite{nishikawa2025nonlinear} showed that nonlinear Transformers can perform inference-time feature learning in in-context learning. 
These works are closely aligned with our motivation: the statistical advantage of deep learning comes not only from large model capacity, but also from the ability to learn task-relevant representations. 

The success of large pre-trained models has also motivated theoretical studies of representation learning and scaling laws. 
Contrastive and self-supervised representation learning have been analyzed as mechanisms for extracting downstream-useful features from auxiliary or unlabeled data \autocite{saunshi2019theoretical,haochen2021provable}. 
Empirical scaling laws have shown that loss often follows predictable power-law behavior as data, model size, or compute increases~\autocite{hestness2017deep,kaplan2020scaling, hoffmann2022training}, and recent theoretical work has attempted to explain such laws through variance-limited regimes, data geometry, kernel spectra, and feature learning~\autocite{bahri2024explaining,bordelon2025feature}. 
Scaling laws have also been studied in transfer and downstream settings \autocite{hernandez2021scaling,lourie-etal-2025-scaling}, where the relation between source data and target tasks becomes essential. 

Our work is situated at the intersection of these theories and the theory of meta-learning. 
Classical and modern meta-learning theory shows that multiple related tasks can reduce the sample complexity of future tasks by learning a shared inductive bias or representation~\autocite{baxter2000model,maurer2016benefit,tripuraneni2021provable}. 
In contrast to most general deep learning theory, which primarily studies single-task approximation, optimization, or generalization, we analyze a meta-representation learning algorithm and prove the achievability of a data scaling law. 
Thus, our result connects the adaptivity and feature-learning viewpoint of deep learning theory with the statistical theory of meta-learning, making explicit how the number of source tasks and the number of samples per task jointly determine downstream sample efficiency.

\end{document}